 \documentclass[10pt,onecolumn]{article}
 \author{
 Arpan Mukherjee  \qquad  Ali Tajer
 \thanks{The authors are with the Department of Electrical, Computer, and Systems Engineering, Rensselaer Polytechnic Institute, Troy, NY 12180.}
 }

\usepackage{ISG_style}
\usepackage{times}
\usepackage[T1]{fontenc}
\usepackage{url}
\usepackage{ifthen}
\usepackage{cite}
\usepackage[cmex10]{amsmath}
\usepackage{amsthm, amssymb, amsfonts,dsfont,mathrsfs,bigints}
\usepackage{graphicx}
\usepackage{algorithm,algorithmic}
\usepackage{appendix}
\usepackage{color}
\usepackage[colorlinks=true,allcolors=blue]{hyperref}
\usepackage{caption}
\usepackage{subcaption}
\newtheorem{theorem}{Theorem}

\newtheorem{definition}{Definition}
\newtheorem{lemma}{Lemma}

\newtheorem{remark}{Remark}

\newcommand*\diff{\mathop{}\!\mathrm{d}}

\DeclareMathOperator*{\argsup}{arg\,sup}
\DeclareMathOperator*{\arginf}{arg\,inf}
\DeclareMathOperator*{\argmax}{arg\,max}
\DeclareMathOperator*{\argmin}{arg\,min}
\DeclareMathAlphabet\mathbfcal{OMS}{cmsy}{b}{n}

\title{\bf \Large Efficient Best Arm Identification in Stochastic Bandits:\\ Beyond  $\beta-$optimality}

\date{}

\begin{document}
\allowdisplaybreaks
\maketitle

\begin{abstract}
This paper investigates a hitherto unaddressed aspect of best arm identification (BAI) in stochastic multi-armed bandits in the fixed-confidence setting. Two key metrics for assessing bandit algorithms are computational efficiency and performance optimality (e.g., in sample complexity). In stochastic BAI literature, there have been advances in designing algorithms to achieve optimal performance, but they are generally computationally expensive to implement (e.g., optimization-based methods). There also exist approaches with high computational efficiency, but they have provable gaps to the optimal performance (e.g., the $\beta$-optimal approaches in top-two methods). This paper introduces a framework and an algorithm for BAI  that achieves optimal performance with a computationally efficient set of decision rules. The central process that facilitates this is a routine for sequentially estimating the optimal allocations up to sufficient fidelity. Specifically, these estimates are accurate enough for identifying the best arm (hence, achieving optimality) but not overly accurate to an unnecessary extent that creates excessive computational complexity (hence, maintaining efficiency). Furthermore, the existing relevant literature focuses on the family of exponential distributions. This paper considers a more general setting of any arbitrary family of distributions parameterized by their mean values (under mild regularity conditions). The optimality is established analytically, and numerical evaluations are provided to assess the analytical guarantees and compare the performance with those of the existing ones. 
\end{abstract}

\section{Introduction}

We consider the problem of best arm identification (BAI) in stochastic multi-armed bandits in the fixed-confidence setting. The bandit instances are assumed to be generated by the single-parameter exponential family (SPEF). In BAI, the objective is to identify the \emph{best} arm (i.e., the arm with the largest mean value) within a pre-specified confidence level with the fewest samples.
Performance optimality and computational efficiency are the two central metrics for assessing and comparing different bandit algorithms.  This paper focuses on an open aspect of BAI in the fixed-confidence parametric setting, which pertains to achieving optimal performance with a computationally efficient algorithm.  In reviewing the existing literature, we will specify these two aspects of the existing algorithms. These will furnish the context to highlight that the current literature on stochastic BAI lacks an algorithm that simultaneously achieves optimal performance and maintains low computational complexity. Subsequently, based on this discussion, we will describe our contributions.

\paragraph{Fixed-confidence versus Fixed-budget.} BAI was first studied as a pure exploration bandit problem in~\cite{Bubeck}. Subsequently, it has been investigated in two broad settings: the \emph{fixed-confidence} setting and the \emph{fixed-budget} setting. The goal in the fixed-confidence setting is to identify the best arm within a specified guarantee on the decision confidence while using as few samples as possible to arrive at a decision. Representative studies in the fixed-confidence setting include~\cite{Gabillon,Kalyanakrishnan2012,pmlr-v49-garivier16a, LinGapE,Jamieson2014,mukherjee2022}. On the other hand, in the fixed-budget setting, the sampling budget is pre-specified. The goal is to minimize the probability of error in the terminal decision. Some representative studies in this setting include~\cite{Bubeck,pmlr-v33-hoffman14,FB_JKS}. Our focus is on the fixed-confidence setting, the literature on which is discussed next.

\paragraph{Bayesian versus Non-Bayesian.} BAI in the fixed-confidence setting can be categorized into Bayesian and non-Bayesian models. Bayesian settings assume a prior distribution on the space of parameters and make arm selection decisions based on the posterior distribution computed from the prior and the observed rewards. In contrast, the non-Bayesian settings do not use posterior sampling for arm selection. Top-two sampling was first introduced in~\cite{russo2016} for the Bayesian setting. The principle of top-two sampling involves dynamically, over time, identifying a {\em leader} and a {\em challenger} as the top arm candidates. Subsequently, the sampling strategy randomizes between these two arms. The top-two Thompson sampling (TTTS) algorithm, proposed and analyzed in~\cite{russo2016,pmlr-v108-shang20a}, involves sampling the posterior for defining the leader and the challenger. Despite the simplicity of TTTS, it faces the computational challenge of repeatedly sampling from the posterior in defining a challenger. To mitigate this,~\cite{pmlr-v108-shang20a} proposed a computationally efficient alternative called the top-two transportation cost (T3C) algorithm. The empirical performance of T3C was further improved using a penalized transportation cost, promoting exploration, in~\cite{jourdan2022}. However, T3C and its improvement only achieve $\beta-$optimality. The nature of a $\beta$-optimality guarantee is as follows: if a $\beta$ fraction of the sampling resources are reserved for the top arm, then these algorithms can determine how to optimality allocate the remaining $(1-\beta)$ fraction among the rest of the arms. Hence, these algorithms are said to be only $\beta-$optimal, where $\beta\in(0,1)$.

In the non-Bayesian setting, \cite{pmlr-v49-garivier16a} has proposed the track-and-stop (TaS) algorithm for BAI, with optimal performance in the asymptote of diminishing probability of error. More investigations on TaS-based algorithms include~\cite{pmlr-v117-agrawal20a} and~\cite{jedra2020optimal}. The TaS sampling strategies, in general, hinge on tracking the optimal allocation of sampling resources over time. Maintaining such allocation in a bandit setting with $K$ arms necessitates solving $K$ equations at each time using the bisection method. Hence, these approaches are generally computationally expensive. It was shown in~\cite{jedra2020optimal} that the tracking procedure could only be performed intermittently at exponentially spaced intervals. This reduces the computational complexity for TaS.
Nevertheless, the approach of~\cite{jedra2020optimal} applies to only linear bandits with Gaussian noise. For stochastic bandits, an asymptotically optimal and computationally efficient alternative was proposed in~\cite{menard2019}, which is based on the lazy sub-gradient ascent algorithm for arm selection. The results of this study on BAI are limited to only Gaussian settings.

To address the computational challenge, the gamification approach was proposed in~\cite{degenne2019,pmlr-v119-degenne20a}. In this approach, BAI is viewed as an unknown two-player game comprising a $\bw$ player and a $\blambda$ player. While the $\bw$ player samples a distribution from the probability simplex consisting of possible allocations, the $\blambda$ player chooses a corresponding bandit instance from the class of instances having a different best arm, with the objective of converging to a saddle point. More recently, a Frank-Wolfe-based algorithm was proposed in~\cite{FW}, which solves BAI using a single iteration of the Frank-Wolfe algorithm. The implementation of this method involves a two-player zero-sum game in each iteration, which requires solving a linear program. To further reduce the computational complexity, the top-two approach has also been used to design algorithms for the non-Bayesian setting. Representative top-two non-Bayesian algorithms include top-two sequential probability ratio test (TT-SPRT)~\cite{mukherjee_SPRT_conf,mukherjee2022,jourdan2022}, top-two expected improvement (TTEI)~\cite{TTEI}, and the empirical best leader with an improved transportation cost (EB-TCI)~\cite{jourdan2022}. While efficient, these algorithms achieve optimal performance only when an instance-dependent parameter $\beta$ is known a priori. Specifically, at their core, these algorithms identify an optimal allocation of the sampling resources among the arms. These algorithms enjoy $\beta$-optimality guarantees. We will discuss and demonstrate empirically that, in some settings, the choice of $\beta$ can critically affect performance (sample complexity).

\paragraph{Parametric versus Non-parametric.} Algorithms for BAI can also be categorized based on whether the bandit instance follows \emph{parametric} or \emph{non-parametric} families of distributions. In the case of the non-parametric family, the upper confidence bound (UCB) based approaches have been investigated (e.g.,~\cite{Gabillon,LinGapE}) and shown to be optimal up to constant factors for the family of sub-Gaussian bandits. More recently, the study in~\cite{pmlr-v117-agrawal20a} has considered the class of distributions satisfying a specific functional boundedness property.
This study proposes a tracking-based sampling strategy along with a likelihood ratio-based stopping rule, which was shown to be asymptotically optimal for this specified class of distributions. For parametric bandits, investigations have focused on the single parameter exponential family (e.g., \cite{pmlr-v49-garivier16a,russo2016,Kaufmann_JMLR, mukherjee2022,jourdan2022}). Some investigations have considered the case of Gaussian bandits with known variances and unknown means, e.g., \cite{TTEI,pmlr-v108-shang20a}. In both of these settings, the top-two sampling strategy is $\beta-$optimal, while the TaS algorithm is asymptotically optimal, despite being computationally expensive.

\begin{figure}[h]
    \centering
        \includegraphics[width=0.5\textwidth]{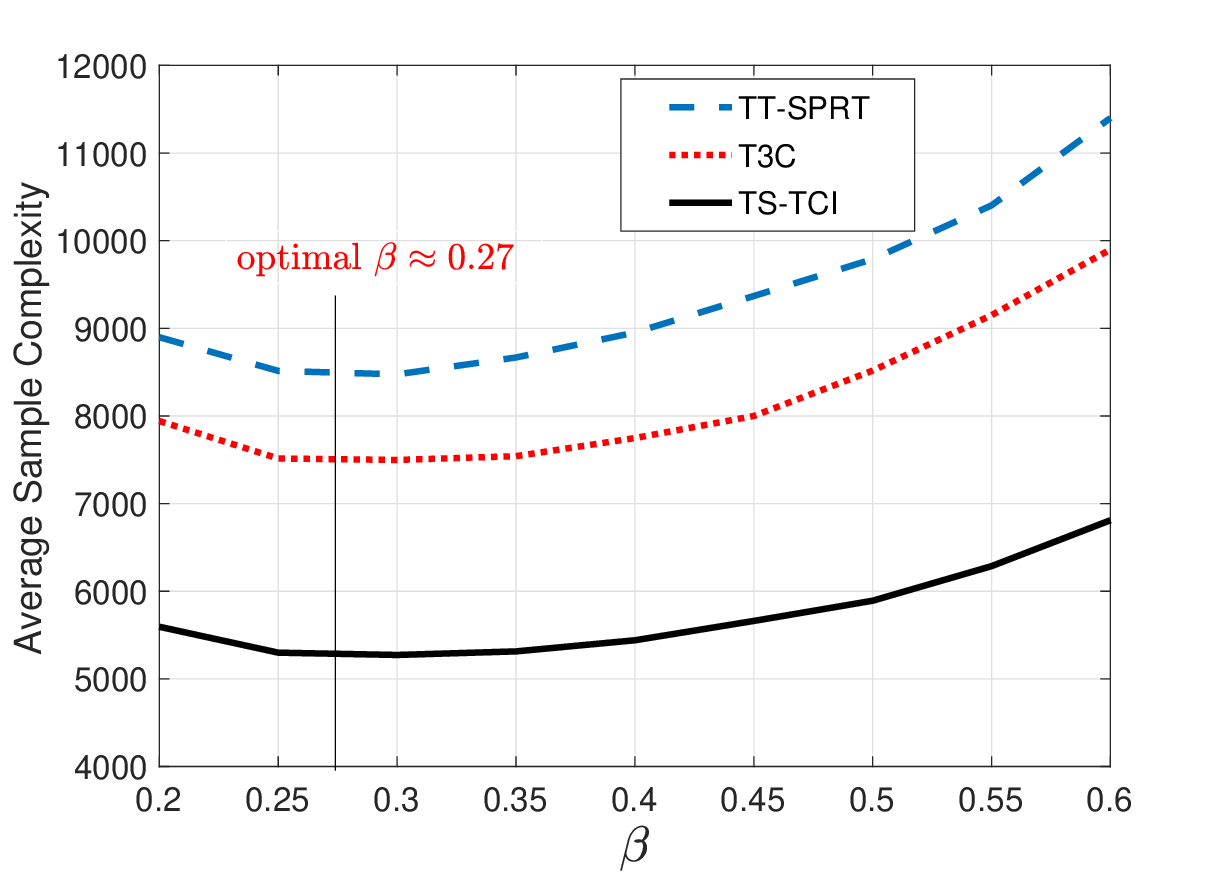} % first figure itself
        \caption{Sample complexity versus $\beta$.}
        \label{fig:1}
    \end{figure}
    
\paragraph{Contributions.} This paper is focused on fixed-confidence, non-Bayesian, and parametric settings. Hence, the directly relevant literature to this scope includes~\cite{pmlr-v49-garivier16a,TTEI,pmlr-v108-shang20a,mukherjee2022,jourdan2022}. As discussed, these studies either achieve optimality at the expense of high computational complexity (e.g., \cite{pmlr-v49-garivier16a,FW}) or maintain computational efficiency but exhibit optimality gaps (e.g., \cite{russo2016,TTEI,pmlr-v108-shang20a,mukherjee2022,jourdan2022}). We propose an algorithm, referred to as the \emph{transportation cost balancing (TCB)} algorithm, that achieves the optimal sample complexity with computationally efficient sampling and decision rules. Specifically, the TCB algorithm avoids solving an optimization problem to form the arm selection decisions in the next round. Furthermore, compared to the existing class of efficient top-two algorithms~\cite{russo2016,pmlr-v108-shang20a,TTEI,jourdan2022,mukherjee2022}, TCB exhibits  optimality in contrast to $\beta-$optimality. Leaping from $\beta$-optimality to optimality, in some bandit instances, leads to significant improvement in sample complexity. To showcase the gap between a $\beta-$optimal solution and an optimal solution, in Figure~\ref{fig:1} we demonstrate how the sample complexity of the $\beta-$optimal solutions vary with respect to $\beta$, demonstrating (i) the sensitivity of the sample complexity to $\beta$, and (ii) a substantial gap (e.g., an order of magnitude) between the optimal and the $\beta-$optimal guarantees. In the face of not knowing the optimal choice of $\beta$, $\beta=0.5$ has been prescribed as a reasonable choice for the top-two algorithms~\cite{russo2016}. However, as expected, this can be noticeably different from the optimal choice in some instances. For instance, Figure~\ref{fig:1} empirically shows the sub-optimality of choosing $\beta=0.5$ in three existing approaches that ensure $\beta-$optimality. In all these cases, the optimal value of $\beta\approx 0.27$.

As the second contribution, we also generalize the probability models to any arbitrary class of parametric models (that satisfy certain regularity conditions). These models subsume the exponential family, which is the only parametric class for which algorithms and performance guarantees are available in the literature. For this generalization, we propose a novel concentration inequality for the generalized log-likelihood ratio (GLLR)-based test statistic for BAI~\cite{pmlr-v49-garivier16a,Kaufmann_JMLR} that holds for any general parametric bandit instance that satisfies a uniform continuity assumption on the divergence measure of the model, and some mild regularity conditions on the arm distributions.

\paragraph{Methodology: Transport Cost Balancing (TCB).} For fixed-confidence BAI, for any bandit instance $\bnu$, the universal lower bound on the average sample complexity is inversely proportional to a problem complexity measure $\Gamma(\bnu)$ (specified later in~(\ref{eq:problem complexity})). Achieving this lower bound is predicated on sequentially determining an optimal allocation of the sampling resources among the arms. While existing optimal algorithms compute these optimal allocations, it is computationally expensive. To mitigate this computational challenge, the top-two sampling rules assign a sampling proportion $\beta\in(0,1)$ to the best arm and then determine the allocation of the remaining $(1-\beta)$ fraction over the rest. The analysis of the top-two methods shows that this facilitates convergence to the $\beta-$optimal allocation. In these methods, the sample complexity is inversely proportional to the transportation cost $\Gamma_{\beta}(\bnu)$, where $\Gamma(\bnu)\geq\Gamma_{\beta}(\bnu)$. Equality holds at the optimal value $\beta^\star$, which depends on the bandit instance and is unknown a priori. 

The central process in our algorithm is a routine for estimating the optimal allocations, including the optimal value of $\beta$, up to a sufficient fidelity that enables confidently identifying the best arm. Our algorithm guides its decisions by balancing transportation costs over time. These decisions lead to efficiently estimating the optimal sampling proportion up to sufficient fidelity. The fundamental advantage of our arm selection rules is that they can track the sampling proportions {\em without} having to compute the optimal sampling proportions at each round. Instead, the sampling proportion is estimated by sampling from the set of under-sampled arms in each round.

\section{Stochastic BAI Model and Assumption}
\label{sec:model and assumption}
\textbf{Stochastic Model.} Denote the class of probability measures defined on any sample space $\Omega\subseteq \R$ by $\mcQ(\Omega)$. Let $\mcP(\Omega)\subset\mcQ(\Omega)$ denote the class of probability measures that are parameterized by their mean values, i.e.,
\begin{align}
    \mcP(\Omega)\triangleq \left\{ \P\in\mcQ(\Omega) : m(\P)\in\Theta\right \}\ ,
\end{align}
where $m(\P)\triangleq \E_{\P}[X]$, and $\E_\P$ denotes the expectation under measure $\P$. Furthermore, define $\mcM\triangleq  \mcP^{\otimes K}(\Omega)$ as the Cartesian product of $K$ sets of measures in $\mcP(\Omega)$. We denote the likelihood function associated with measure $\P$ by $\pi_\P$. We make the following assumptions on this stochastic model.
\begin{enumerate}
    \item $\Theta\subseteq \R$ is a compact parameter space, which is {\em known} to the learner.  
    \item The likelihood functions $\pi_\P$ are continuous and twice-differentiable in $m(\P)$ for every $\P\in\mcP(\Omega)$. Furthermore, the log-likelihood function $\log\pi_{\P}(\cdot\med m(\P))$ is concave in $m(\P)$ for any $m(\P)\in\Theta$.
    \item All distributions in $\mcP(\Omega)$ have the {same support} $\Omega$.
    \item All the distributions in $\mcP(\Omega)$ have \emph{finite} third moments, i.e., $\E_{\P}[|X-m(\P)|^3]<+\infty$ for every $\P\in\mcP(\Omega)$.
    \item For any $\theta,\theta^\prime\in\Theta$, we denote the Kullback-Leibler (KL) divergence between $\P_{\theta}\in\mcP(\Omega)$ and $\P_{\theta^\prime}\in\mcP(\Omega)$ by $D_{\sf KL}(\P_{\theta}\|\P_{\theta^\prime})$. Keeping one argument fixed, the KL divergence is assumed to be uniformly continuous in the second argument.
    \item For any likelihood function $\pi_\P$, let us define the Fisher Information (FI) measure as
    \begin{align}
        \mcI_{\P}(\theta)\;\triangleq\;-\E_{\P}\left[\frac{\partial^2}{\partial\theta^2}\log\pi_{\P}(\cdot\med\theta)\right]\ .
    \end{align}
    We assume that $\mcI_{\P}(\theta)<+\infty$ for all $\theta\in\Theta$ and $\P\in\mcP(\Omega)$.
    \item There exists $\sigma^2>0$ such that for any $\P\in\mcP(\Omega)$,
    \begin{align}
    \label{eq:A4}
        \frac{\partial^2\log\pi_{\P}(x\med\theta)}{\partial\theta^2}\leq -\sigma^2\ ,\quad \forall x\in\Omega,\;\forall\theta\in\Theta\ .
    \end{align}
    \item For any $\theta,\theta^\prime\in\Theta$, we assume that
    \begin{align}
    \label{eq:quick_condition}
        \E_{\P}\left [ \left\lvert \log\frac{\pi_{\P}(X\med\theta)}{\pi_{\P}(X\med\theta^\prime)}\right\rvert^3\right ]\;<\;+\infty\ .
    \end{align}
\end{enumerate}
Assumption 1 is needed for designing the estimator for the unknown arm means from the observed rewards and is a common assumption in the BAI literature~\cite{degenne2019,FW}. Assumptions 2 and 3 ensure the existence of the maximum likelihood estimates (MLEs) for the parameters of interest (i.e., the mean values) and that the KL divergence measures between any pair of distributions in $\mcM$ are finite. Assumption 4 ensures the almost sure convergence of the sample mean to the ground truth if each arm is sampled sufficiently often. Assumption 5 depicts the uniform continuity of the KL divergence in each argument. We emphasize that this is a mild assumption, and BAI becomes significantly hard without this assumption. Specifically, there exists an impossibility result~\cite{pmlr-v117-agrawal20a}, which states that BAI is impossible for classes of measures that are ``KL right dense'', i.e., if the mean values between two measures in the class can be arbitrarily large, even though the KL divergence between the measures is bounded by an arbitrarily small value. Assumption 6 implies the finiteness of the FI measure for all distributions in $\mcP(\Omega)$. The condition in~(\ref{eq:A4}) specifies a bound on the second-order derivative of the log-likelihood function over the parameter space $\Theta$. For instance, for the exponential family of distributions, (\ref{eq:A4}) is equivalent to the variance being bounded away from zero. Finally, the condition in~(\ref{eq:quick_condition}) implies that the log-likelihood ratio corresponding to parameters $\theta$ and $\theta^\prime$ has a finite third moment.

\paragraph{Bandit Model.} Consider a $K$-armed stochastic bandit. The rewards of arm $i\in[K]\triangleq\{1,\cdots,K\}$ are generated from $\P_i\in\mcP(\Omega)$.We define $\mu(i)\triangleq  m(\P_i)$, and denote the likelihood function associated with $\P_i$ parameterized by the mean value $\mu(i)$ by $\pi_i(\cdot\med\mu(i))$. Accordingly, we define the bandit instance $\bnu\triangleq [\P_1,\dots,\P_K]$. Furthermore, define $\mcM\triangleq  \mcP^{\otimes K}(\Omega)$ as the Cartesian product of $K$ sets of measures in $\mcP(\Omega)$. Finally, let $(\mcM,D_{\sf TV})$ denote the metric space of the distributions in the set $\mcM$ endowed with the total variation distance metric~$D_{\sf TV}$.

\paragraph{Sequential Decisions.} At each round $t\in\N$, the learner chooses an action $A_t\in[K]$, and receives a reward $X_t\sim\P_{A_t}$. We denote the sequence of actions, the corresponding rewards, and the filtration generated by the sequence of actions and rewards by the ordered sets
\begin{align}
    &\mcA_t\triangleq \{A_s : s\in[t]\}\ ,\nonumber\\ &\mcX_t\triangleq \{X_s : s\in[t]\}\ ,\nonumber\\ \text{and}\qquad &\mcF_t\triangleq \{A_1,X_1,\cdots,A_t,X_t\}\ .
\end{align}
Denote the set of rewards obtained by selecting an arm $i\in[K]$ up till time $t$ by
\begin{align}
    \mcX_t^i\triangleq \{X_s : s\in[t],\;A_s = i\}\ .
\end{align}
The objective of the learner is to identify the best arm, which is assumed to be \emph{unique}, and is defined as the arm with the largest mean, i.e.,
\begin{align}
    a^\star\triangleq \argmax_{i\in[K]}\; \mu(i)\ .
\end{align}
For the algorithm design, we use information projection measures defined as follows. For any measure $\P\in\mcP(\Omega)$ and $x\in\R$, we define 
\begin{align}
\label{eq:d_U}
    &d_{\sf U}(\P,x)\;\triangleq\; \inf\limits_{\Q\in\mcP(\Omega) :\;  m(\Q)\;\leq\; x}\; D_{\sf KL} (\P \| \Q)\ ,\\ \text{and}\qquad &d_{\sf L}(\P,x)\;\triangleq\; \inf\limits_{\Q\in\mcP(\Omega):\; m(\Q) \;\geq \; x}\; D_{\sf KL}(\P \| \Q)\ .
\label{eq:d_L}
\end{align}
Specifically, the information measure $d_{\sf U}(\P,x)$ for any distribution $\P\in\mcP(\Omega)$ and $x\in\R$ is the minimum KL divergence between $\P$ and any distribution with mean {\em at most} $x$. Similarly, $d_{\sf L}(\P,x)$ measures the KL divergence between $\P$ and any distribution with mean {\em at least} $x$. In the fixed confidence setting, the goal is to identify the best arm with a pre-specified level of confidence while minimizing the number of samples in making the decision. Let $\tau$ denote an $\mcF$-adapted stopping time, i.e., $\{\tau=t\}\in\mcF_t$ for every $t\in\N$. Corresponding to the stochastic stopping time $\tau$, let $\hat A_{\tau}$ denote the terminal decision of the learner. The $\delta$-PAC objective of the learner is formalized next. 
\begin{definition}[$\delta-$PAC]
A BAI algorithm is $\delta-$PAC 
if the algorithm has a stopping time $\tau$ adapted to $\{\mcF_t:t\in\N\}$, and at the stopping time with the terminal decision $\hat A_\tau\in[K]$ it ensures 
\begin{align}
    \P_{\bnu}\{\tau<+\infty,\;\hat A_\tau = a^\star\} > 1 - \delta\ ,
\end{align}
where $\P_{\bnu}$ denotes the probability measure induced by the interaction of the BAI algorithm with the bandit instance $\bnu$.
\end{definition}

\section{Transportation Cost Balancing Algorithm}

In this section, we specify (i) a stopping rule that decides when to stop collecting samples and form a confident decision about the best arm, (ii) an arm selection rule that guides the order of sampling arms over time, and (iii) an estimation routine that aims to learn the unknown model parameters of interest (e.g., the mean values). {We use different estimators for the arm selection and stopping rules. Specifically, we use the maximum likelihood estimate (MLE) for the stopping rule, whereas we use the sample mean to estimate the arm selection strategy. The central statistic that guides all three decisions in the TCB algorithm is the GLLR. To formalize the GLLRs, we define $K(K-1)$ hypotheses $\{\mcH_{i,j}:i\in[K]\}$ such that for all $i\neq j$
 \begin{align}
     \mcH_{i,j}\; : \; \mu(i)\;\geq\mu(j)\ .
 \end{align}
Next, we formalize the GLLR test statistic for performing these hypothesis tests, which has been adopted in a wide range of investigations on parametric BAI~\cite{pmlr-v49-garivier16a,Kaufmann_JMLR,TTEI,pmlr-v108-shang20a,degenne2019,FW,mukherjee2022}. At any time $t\in\N$ and for any arm $i\in[K]$, and based on the samples available from this arm, we denote the MLE of $\mu(i)$ projected on $\Theta$ by $\mu_t(i)$, i.e.,
\begin{align}
\label{eq:MLE}
    \mu_t(i)\triangleq \argmax\limits_{\mu\in\Theta}\; \sum\limits_{s\in[t]} \log \pi_i(X_s\med \mu)\cdot\mathds{1}_{\{A_s = i\}}\ ,
\end{align}
where $\mathds{1}$ denotes the indicator function. 
Furthermore, for any two parameters $\theta,\theta^\prime\in\Theta$, and for any arm $i\in[K]$, let us define
\begin{align}
    d_i(\theta\|\theta^\prime)\;\triangleq\;D_{\sf KL}(\pi_i(\cdot\med\theta)\|\pi_i(\cdot\med\theta^\prime))\ .
\end{align}
Accordingly, for any pair of arms $(i,j)\in[K]\times[K]$, let us define
\begin{align}
    \Lambda_t(i,j)\;\triangleq\;\min\limits_{\brho\in\R^K : \rho(i)\leq\rho(j)}\left\{ T_t(i)d_i(\mu_t(i)\|\rho(i)) + T_t(j)d_j(\mu_t(j)\|\rho(j))\right\}\mathds{1}_{\{\mu_t(i)\geq\mu_t(j)\}}\ ,
    \label{eq:GLLR}
\end{align}
where we have defined 
\begin{align}
    T_t(i)\triangleq \sum_{s=1}^t\mathds{1}{\{A_s=i\}}\ ,
\end{align}
as the counter for the number of times arm $i\in[K]$ is chosen up to time~$t$.}

\paragraph{GLLR-based Stopping Rule.} We specify a GLLR-thresholding stopping criterion, which compares the GLLR statistic against a time-varying threshold. When the GLLR statistic exceeds the threshold, the algorithm stops collecting more samples and forms a decision about the best arm. Let us denote the maximum likelihood (ML) decision at time $t$ about the top arm by $a_t^{\sf top}$, i.e.,
\begin{align}
    a_t^{\sf top}\;\in\;\argmax_{i\in[K]}\; \bar\mu_t(i)\ ,
\end{align}
where $\bar\mu_t(i)$ denotes the sample mean of arm $i\in[K]$, projected on to the parameter space $\Theta$. Our stopping criterion is based on sufficiently distinguishing between the ML decision $a_t^{\sf top}$ and the most likely contender compared to the best arm, which we refer to as the \emph{challenger}. This challenger arm at time~$t$ is the arm closest to the top arm in a GLLR sense, and it is specified by
\begin{align}
    a_t^{\sf ch}\;\in\; \argmin_{i\in[K]\setminus\{i:\mu_t(i)<\mu_t(a_t^{\sf top})\}}\; \Lambda_t(a_t^{\sf top},i)\ .
\end{align}
{In other words, $a_t^{\sf ch}$ denotes the arm that is the top contender to the best arm $a_t^{\sf top}$, where the comparison is made using the likelihood ratio between the arms.} The stopping rule compares the GLLR between $a_t^{\sf top}$ and $a_t^{\sf ch}$ and stops collecting samples when the GLLR exceeds a threshold. The threshold depends on the level of confidence $\delta$ required on the final decision, and it is denoted by $\beta_t(\delta)$. The stopping rule is stated next.
\begin{align}
\label{eq:stop}
    \tau\;\triangleq\; \inf\left\{ t\in\N : \Lambda_t(a_t^{\sf top},a_t^{\sf ch}) > \beta_t(\delta)\right\}\ .
\end{align}
The threshold $\beta_t(\delta)$ is specified in Theorem~\ref{theorem: delta-PAC} to ensure the $\delta$-PAC guarantee on the decision. Next, we delineate the arm selection rules. For this purpose, we first formalize the \emph{problem complexity} measure, which quantifies the hardness of identification in a BAI instance. Specifically, the problem complexity captures the {\em minimum} distance between the given bandit instance and any other bandit instance with a different best arm. Clearly, the smaller the problem complexity, the larger the number of samples required for identification. We also state an equivalent representation of the problem complexity, which motivates our arm selection strategies.

\paragraph{Problem complexity.} Consider a bandit instance $\bnu = [\P_1,\cdots,\P_K]$ with the top arm $a^\star$. Given $a^\star$, we define an alternative set of bandit instances $\bar\bnu = [\bar\P_1,\cdots,\bar\P_K]$ in which the top arm is not $a^\star$. Specifically,
\begin{align}
    {\sf alt}(a^\star)\triangleq \{\bar\bnu\in\mcM : m(\bar\P_{a^\star})\leq \max\limits_{i\neq a^\star}m(\bar\P_i)\}\ .
\end{align}
Subsequently, given $a^\star$ and ${\sf alt}(a^\star)$, under the weight vector $\bw = [w_1,\cdots,w_K]\in\Delta^K$, where $\Delta^K$ represents the $K$-dimensional probability simplex, we define the problem complexity associated with $\bnu$ as the smallest weighted KL divergence from $\bnu$ to the set ${\sf alt}(a^\star)$. Specifically,
\begin{align}
\label{eq:problem complexity w}
    \Gamma(\bnu,\bw)\;\triangleq\; \inf\limits_{\bar\bnu\in{\sf alt}(a^\star)}\; \sum\limits_{i\in[K]} w_i\;D_{\sf KL}(\P_i\|\bar\P_i)\ .
\end{align}
{Given any weight vector $\bw\in\Delta^K$, $\Gamma(\bnu,\bw)$ captures the hardness of distinguishing $\bnu$ from the closest alternate bandit instance, where the divergence between the arms is weighted by $\bw$.}
Finally, we define the problem complexity associated with the bandit instance $\bnu$ as
\begin{align}
\label{eq:problem complexity}
    \Gamma(\bnu)\;\triangleq\; \sup\limits_{\bw\in\Delta^K}\; \Gamma(\bnu,\bw)\ .
\end{align}
{It can be readily verified that $\Gamma(\bnu)$ captures the {\em maximum} hardness in distinguishing $\bnu$ from the closest bandit instance.} Accordingly, we define the maximizer weight vector as
\begin{align}
\label{eq: optimal_proportions}
    \bw(\bnu)\;\triangleq\; \argsup\limits_{\bw\in\Delta^K}\; \Gamma(\bnu,\bw)\ .
\end{align}
{The weight vector $\bw(\bnu)$ characterizes the {\em optimal} allocation in which to sample arms, such that the sample complexity for BAI for the bandit instance $\bnu$ is minimized}. Next, we provide an equivalent representation of the problem complexity measure, which facilitates analyzing its key properties.
\begin{lemma}
\label{lemma:simplified problem complexity}
    The problem complexity $\Gamma(\bnu,\bw)$ defined in~(\ref{eq:problem complexity w}) can be equivalently expressed as
    \begin{align}
        \Gamma(\bnu,\bw)\;=\; \min\limits_{i\neq a^\star}\inf\limits_{x\in[\mu(i),\mu(a^\star)]}\quad \Big \{w_{a^\star} d_{\sf U}(\P_{a^\star},x) + w_i d_{\sf L}(\P_i,x)\Big \}\ .
        \label{eq:lemma 2}
    \end{align}
\end{lemma}
\begin{proof}
\label{lemma:problem complexity 0}
    The proof follows a similar line of arguments as~\cite[Lemma 3]{pmlr-v49-garivier16a}. For completeness, we provide the proof in Appendix~\ref{proof: simplified problem complexity}.
\end{proof}
\noindent{The expression for $\Gamma(\bnu,\bw)$ in~(\ref{eq:lemma 2}) involves an inner minimization, which is a weighted combination of divergence measures from $\bnu$ to an alternate bandit instance. Note that the inner minimization only depends on the divergence measures for the best arm $a^\star$ and any other arm $i\neq a^\star$. Furthermore, the outer minimization acts over all the other arms $i\neq a^\star$, establishing that $\Gamma(\bnu,\bw)$ is the weighted divergence measure between the bandit instance $\bnu$ and the {\em closest} alternate bandit instance}. 

\paragraph{Transportation cost balancing (TCB).} 
Designing the arm selection rule consists of two key components. The first component ensures that none of the arms remain under-explored. Specifically, the objective in this phase is to ensure that we have a reasonable estimate of each arm's mean value, such that our estimates converge to the true mean values if the arm selection rule is allowed to collect samples without stopping. While estimating the mean values is not the goal in BAI, our arm selection rule performs estimation as an intermediate step to form a confident decision about the best arm. The second component of the arm selection rule is to track the optimal proportion $\bw(\bnu)$ of arm selections defined in~(\ref{eq: optimal_proportions}), which ensures that we minimize the average sample complexity. For this, TaS~\cite{pmlr-v49-garivier16a} proposes to compute the optimal sampling proportions at the current mean estimates and track the estimated sampling proportions. However, this is computationally expensive and requires solving $K$ equations using the bisection method in each round. To circumvent this, we propose a simple sampling mechanism focusing on sampling from the set of under-sampled arms at each instant, in which case, the sampling strategy can converge to the optimal sampling proportions $\bw(\bnu)$ asymptotically. Next, we describe the sampling rule that combines these two components.

\paragraph{Under-explored Arms.} At any time $t$, the sampling rule defines a set of \emph{under-explored} arms as
\begin{align}
    \mcU_t \; \triangleq \; \left\{i\in[K] : T_t(i) \leq \left\lceil \sqrt{t/K}\right\rceil\right\}\ .
\end{align}
If the set of under-explored arms is non-empty, indicating that some of the arms are under-explored, the arm selection strategy selects the arm that is sampled the least. {Otherwise, when there are no under-explored arms to sample, the goal is to devise a sampling strategy that {\em estimates} the optimal $\beta$, i.e., the optimal allocation of the best arm. For this purpose, our sampling strategy aims to ensure the almost sure convergence of the sampling proportions to the optimal allocation $\bw(\bnu)$. We show in Lemma~\ref{lemma:con_alloc_3} (Appendix~\ref{proof: convergence in allocation}) that this objective is achieved by any sampling rule that eventually {\em always} samples from the set of under-sampled arms, i.e., the set of arms which are sampled less number of times compared to the optimal allocation. Consequently, we devise a sampling strategy that (eventually) always samples from the set of under-sampled arms.} At time $t$, if $\mcU_t = \emptyset$, this arm selection rule leverages the MLEs of the arm means to compute an empirical estimate of $\Gamma(\bnu,\bw)$ defined in Lemma~\ref{lemma:simplified problem complexity}. Based on this, the goal is to select the next arm in a way that maximizes the estimate. We show in Appendix~\ref{proof: convergence in allocation} that this is equivalent to sampling from the set of under-sampled arms.

\paragraph{Transport cost-based Estimation of Allocation.}  We provide a few measures that are used to delineate the arm selection rule. Let us denote the distribution of arm $i\in[K]$ parameterized by the sample mean $\bar\mu_t(i)$ projected on $\Theta$ by $\P_{t,i}$. For any arm $i\in[K]$, define the interval $I_{t,i}\triangleq [\bar\mu_t(i),\bar\mu_{t}(a_t^{\sf top})]$, which specifies the interval for minimization in the empirical problem complexity defined in~\eqref{eq:Gamma_t} next. Based on these, we define the \emph{minimum transportation cost}~\cite{pmlr-v108-shang20a}, as the minimum weighted combination of divergence measures $d_{\sf U}$ and $d_{\sf L}$ defined in~(\ref{eq:d_U}) and~(\ref{eq:d_L}) of the top arm $a_t^{\sf top}$ and any other arm $i\neq a_t^{\sf top}$ as follows.
\begin{align}
\label{eq:Gamma_t}
    \Gamma_t(\bw)\;\triangleq\; \min_{i\in[K]\setminus\{a_t^{\sf top}\}}\min\limits_{x\in I_{t,i}}\; \;\Bigg\{ w_{a_t^{\sf top}} d_{\sf U}(\P_{t,a_t^{\sf top}},x) + w_i d_{\sf L}(\P_{t,i},x)\Bigg \}\ .
\end{align}
$\Gamma_t(\bw)$ in~(\ref{eq:Gamma_t}) specifies the minimum cost of transporting the currently estimated bandit instance $\bnu_t\triangleq [\P_{t,1},\cdots,\P_{t,K}]$ to an alternate bandit instance for which the best arm is not $a_t^{\sf top}$. Even though we optimize for the information projection measures $d_{\sf U}$ and $d_{\sf L}$ in $\mcP(\Omega)$, it is worth noting that we may have additional knowledge about the class of bandit instances over which we want to optimize. For example, we may restrict $\mcP(\Omega)$ to be the set of distributions in the single parameter exponential family with the cumulant generating function $b : \Theta\mapsto\R$, in which case, we have an explicit closed-form expression for $\Gamma_t(\bw)$~\cite{pmlr-v108-shang20a,mukherjee2022,jourdan2022}. However, our analysis holds for a general class of bandit instances satisfying Assumptions 1-8, and a closed-form expression for $\Gamma_t(\bw)$ can be derived based on the bandit instance in consideration.
{The TCB arm selection rule is based on a {\em look-ahead} distribution, which navigates the arm selection routine to sample arms in a way that increases the empirical problem complexity. To this end, let us define the look-ahead distribution over the arms $\bw^\prime_t\in\Delta^K$ such that
\begin{align}
\label{eq: look ahead}
        &w^\prime_t(i)\triangleq  \left\{
	\begin{array}{ll}
	\frac{T_t(i)}{t} + r_t, & \mbox{if} \;\; i = a_t^{\sf top}\\
	\frac{T_t(i)}{t} - \frac{r_t}{K-1} , & \mbox{otherwise}
	\end{array}\right. \ ,
\end{align}
where $\{r_t : t\in\N\}$ is a sequence of positive real numbers satisfying $\limsup_{t\uparrow\infty} r_t = 0$, such that it preserves the property that $\bw^\prime_t\in\Delta^K$.} 
\begin{remark}
    Note that any choice of the sequence $\{r_t : t\in\N\}$ satisfying $\limsup_{t\uparrow\infty} r_t = 0$ is sufficient for the performance guarantees presented in Section~\ref{sec:performance guarantees}. However, it is possible that different choices of $\{r_t : t\in\N\}$ promote convergence in empirical allocations to the optimal one at different rates. However, in this investigation, we are not interested in the rate of convergence in the allocation estimates. Rather, we require the estimates to be sufficiently close to the optimal values at stopping, which ensures $\delta-$PAC BAI. Furthermore, since we investigate the {\em asymptotic} sample complexity in Theorems~\ref{theorem: SC upper bound}, it suffices for the sequence $\{r_t : t\in\N\}$ to converge to $0$ asymptotically in $t$. 
\end{remark}

Based on the equivalent form of the problem complexity in Lemma~\ref{lemma:simplified problem complexity}, the goal of the sampling strategy is to maximize the lowest information measure $\Gamma_t(\bw)$. This ensures that we sample the under-sampled arms in each round and move closer toward the optimal sampling proportion. 
To this end, let us define the arm with the lowest information measure~as
\begin{align}
\label{eq:a_t^K}
    a_t^{\min}\;\in\;\argmin_{i\in[K]\setminus\{i : \mu_t(i)<\mu_t(a_t^{\sf top})\}}\; \min\limits_{x\in I_{t,i}}\; \;\left\{ \frac{T_t(a_t^{\sf top})}{t} d_{\sf U}(\P_{t,a_t^{\sf top}},x) + \frac{T_t(i)}{t} d_{\sf L}(\P_{t,i},x)\right \}\ .
\end{align}
In~(\ref{eq:a_t^K}), $a_t^{\min}$ denotes the arm that minimizes the estimate of the problem complexity at time $t$. We note that the arms $a_t^{\min}$ and $a_t^{\sf ch}$ are generally distinct. Even though they occasionally might refer to the same arm,  it is not the case. Specifically, given a class of measures, the transportation cost evaluated at the current sampling proportions and the scaled GLLR $\frac{1}{t}\Lambda_t(a_t^{\sf top},a_t^{\sf ch})$ defined in~\eqref{eq:GLLR} may have a similar form. However, the transportation cost is evaluated with measures parameterized by the current sample mean, while the GLLR is evaluated with measures parameterized by the constrained MLE. If these estimates are significantly different, we may have different candidates as the arms $a_t^{\min}$ and $a_t^{\sf ch}$. On the other hand, for special classes, such as the exponential family, the sample mean and the MLE are the same, in which case we may have $a_t^{\min}=a_t^{\sf ch}$. Note that to increase the lowest information measure, we should select either the current best arm $a_t^{\sf top}$ or the arm with the lowest information measure $a_t^{\min}$. Based on the above definitions, the arm selection for TCB is carried out as follows.
{
\begin{align}
\label{eq: sampling rule}
        &A_{t+1}\triangleq  \left\{
	\begin{array}{ll}
	\argmin\limits_{i\in\mcU_t} T_t(i), & \mbox{if} \;\; \mcU_t\neq\emptyset\\
	a_t^{\sf top} , & \mbox{if}\;\;\Gamma_{t}(\bw^\prime_t) > \Gamma_{t}\left(\frac{1}{t}\bT_t\right)\; \text{and}\;\mcU_t=\emptyset \\
	a_t^{\min} , & \mbox{if}\;\;\Gamma_{t}(\bw^\prime_t) < \Gamma_{t}\left(\frac{1}{t}\bT_t\right)\;\text{and}\;\mcU_t=\emptyset \\
	\end{array}\right. \ ,
\end{align}
where $\bT_t\triangleq [T_t(1),\cdots,T_t(K)]$.} The complete procedure is presented in Algorithm~\ref{algorithm:TT_SPRT}. 

%%%%%%%%% algorithm %%%%%%%%
		\begin{algorithm}[h]
        \small
%		\algsetup{linenosize=\small}
%		\setstretch{0.85}
		\caption{Transportation cost balancing (TCB)}
		\label{algorithm:TT_SPRT}
		
 		%\small
 		\begin{algorithmic}[1]
%			\textbf{Input:} $\beta$
            \STATE \textbf{Initialize:} $t=0$, $\mcU_t = [K]$, $\mu_t(i) = 0\;\forall\;i\in[K]$, $T_t(i) = 0\;\forall\;i\in[K]$, $\Lambda(a_t^{\sf top},a_t^{\sf ch})=0$, $\beta_t(\delta) = 0$\\
			\WHILE{$\Lambda(a_t^{\sf top},a_t^{\sf ch}) \leq \beta_t(\delta)$}
			    \STATE $t\leftarrow t + 1$\\
			    %Sample $D_n\sim {\sf Bern}(\beta)$\\
			    \STATE Select an arm $a_{t}$ specified by~(\ref{eq: sampling rule}) and obtain reward $X_t$\\
			    \STATE Update $\mu_t(a_t)$ and $T_t(a_t)$ using~(\ref{eq:MLE})\\
			    \STATE $a_t^{\sf top} \leftarrow \argmax\limits_{i\in[K]}\bar\mu_t(i)$\\
			    \STATE Compute $a_t^{\min}$ using~(\ref{eq:a_t^K})\\
			    \STATE For every $i\in[K]$, compute $w_t^\prime(i)$ using~(\ref{eq: look ahead})\\
			    \STATE Compute $\Gamma_t(\bw_t^\prime)$ and $\Gamma_t(\frac{1}{t}\bT)$ using~(\ref{eq:Gamma_t})\\
			    \STATE Compute $\Lambda_t(a_t^{\sf top},i)$ for every $i\in[K]\setminus\{a_t^{\sf top}\}$\\
			    \STATE $a_t^{\sf ch} \leftarrow \argmin\limits_{i\in[K]\setminus\{a_t^{\sf top}\}} \Lambda_t(a_t^{\sf top},i)$\\
			    \STATE Update $\beta_t(\delta)$ using~(\ref{eq:stop})
			\ENDWHILE
			\STATE \textbf{Output:} Top arm $a_t^{\sf top}$
 		\end{algorithmic}
	\end{algorithm}
% \newpage
\paragraph{Improved Transportation Cost Balancing (ITCB).} A recent study by~\cite{jourdan2022} shows that an additional exploration penalty based on the number of times that each arm is chosen improves the empirical performance of top-two algorithms, albeit achieving the same asymptotic optimality guarantee. Specifically, \cite{jourdan2022} proposes an additive penalty $\log(T_t(i))$ to the GLLRs for each arm $i\in[K]\setminus\{a_t^{\sf top}\}$ to promote further exploration of under-explored arms. Motivated by this observation, we also devise a modified sampling rule that achieves the same optimality guarantee as TCB, with improved empirical performance. To formalize the modified sampling rule, we begin by defining the lowest \emph{penalized} information measure as:
{
\begin{align}
    \Phi_t(\bw)\; \triangleq \; \min_{i\in[K]\setminus\{a_t^{\sf top}\}}\;\left\{ \min\limits_{x\in I_{t,i}}\; \Big\{ w_{a_t^{\sf top}} d_{\sf U}(\P_{t,a_t^{\sf top}},x) + w_{i} d_{\sf L}(\P_{t,i},x)\Big \} + \frac{\log( tw_i)}{t}\right \} \ ,
\end{align}}
for any $\bw\in\Delta^K$. Furthermore, let us define the arm having the lowest penalized information measure as
\begin{align}
    b_t^{\min}\; \triangleq  \;\argmin_{i\in[K]\setminus\{a_t^{\sf top}\}}\left \{\; \min\limits_{x\in I_{t,i}} \Bigg\{ \frac{T_t(a_t^{\sf top})}{t} d_{\sf U}(\P_{t,a_t^{\sf top}},x) + \frac{T_t(i)}{t} d_{\sf L}(\P_{t,i},x)\Bigg \} + \frac{\log(T_t(i))}{t}\right \}\ ,
\end{align}
%where $\bw\in\Delta^K$. 
Based on this definition, the modified sampling rule in ITCB is specified next.
{
\begin{align}
\label{eq: sampling rule 2}
        &A_{t+1}\triangleq \left\{
	\begin{array}{ll}
	\argmin\limits_{i\in\mcU_t} T_t(i), & \mbox{if} \;\; \mcU_t\neq\emptyset\\
	b_t^{\min} , & \mbox{if}\;\;\Phi_t(\bw_t^\prime) > \Phi_t \left(\frac{1}{t}\bT_t\right)\;\text{and}\;\mcU_t=\emptyset \\
	a_t^{\sf top} , & \mbox{if}\;\;\Phi_t(\bw_t^\prime) < \Phi_t \left(\frac{1}{t}\bT_t\right)\;\text{and}\;\mcU_t=\emptyset \\
	\end{array}\right. \ .
\end{align}}
{Note that main difference in the selection rules~(\ref{eq: sampling rule}) and~(\ref{eq: sampling rule 2}) lies in the cost function $\Phi_t$ which has an additional penalty compared to $\Gamma_t$, that promotes additional exploration. The ITCB sampling rule follows the same principle as TCB, with the difference of a penalized cost function.}
\section{Performance Guarantees}
\label{sec:performance guarantees}

This section provides the main results on the performance of TCB and ITCB algorithms. There are two key results that we are interested in proving. First, we want to show that the TCB and ITCB algorithms satisfy the $\delta-$PAC guarantee on the decision confidence. Next, we show that the average sample complexities of these algorithms match the known information-theoretic lower bound asymptotically. To this end, we will prove a few properties of TCB and ITCB, which will collectively establish asymptotic optimality in terms of the average sample complexity. We start by stating the guarantee on the probability of error and characterizing $\beta_t(\delta)$ such that the algorithms are $\delta-$PAC. To formalize this result, we define a few quantities that characterize the stopping threshold $\beta_t(\delta)$. We denote the FI measure corresponding to arm $i\in[K]$ evaluated at the current MLE $\mu_t(i)$ by $\mcI_i(\mu_t(i))$. For any $i\in[K]$, let us define
\begin{align}
    \bar V_t(i)\;\triangleq\; -\sum\limits_{s\in[t]:A_s = i} \bigg(\frac{\partial^2}{\partial\theta^2} \log\pi_i(X_s\med\theta)\bigg)_{\theta = \mu_t(i)}\ ,
\end{align}
which represents the second-order derivative of the log-likelihood function of arm $i\in[K]$, evaluated at the ML estimate $\mu_t(i)$. Accordingly, for any $\varepsilon\in\R_+$ and $i\in[K]$, we define %\AT{change the notation of $W_t(\varepsilon,i)$. Update the rest too.}
\begin{align}
    W_t(\varepsilon,i)\;\triangleq\; \displaystyle\bigintsss_{\Omega^{\otimes T_t(i)}}\log\left(1 - 2Q\left (\varepsilon\sqrt{\bar V_t(i)}\right )\right)\displaystyle\prod\limits_{s\in[t]:A_s = i}\pi_i(X_s\med\mu_t(i))\diff\mcX_t^i\ ,
\end{align}
where $Q(x)$ denotes the $Q$ function.  Furthermore, let us define
\begin{align}
    W_t(\varepsilon)\;\triangleq\; \max\limits_{i\in[K]}\log\mcI_i(\mu_t(i)) - 2\min\limits_{i\in[K]} W_t(\varepsilon,i)\ .
    \label{eq:W_t}
\end{align}
 Note that as $t\rightarrow\infty$, $\mcI_i(\mu_t(i))$ converges to the FI measure under the true parameter $\mu(i)$, and the second term in~\eqref{eq:W_t} converges to $0$, given that each arm is sampled sufficiently often. To see why this is true, recall that according to Assumption 7, we have $\bar V_t(i)\geq T_t(i)\sigma^2$ for any $i\in[K]$, which can become infinitely large if the arm $i\in[K]$ is sampled sufficiently often. Hence, $W_t(\varepsilon)$ converges to the maximum FI measure for the distributions in the bandit instance $\bnu$. Furthermore, for any $\varepsilon\in\R_+$ let us define
 \begin{align}
     \varepsilon_t\;\triangleq\; \max\limits_{i\in[K]} \Big\{\max\left\{d_i(\mu_t(i)\|\mu_t(i)-\varepsilon),d_i(\mu_t(i)\|\mu_t(i)+\varepsilon)\right\}\Big\}\ .
     \label{eq:varepsilon_t}
 \end{align}
 Note that the quantity $\varepsilon_t$ can be made arbitrarily small by choosing a sufficiently small $\varepsilon$ as a consequence of the uniform continuity of the KL divergence measures in Assumption 5. Subsequently, we state the choice of $\beta_t(\delta)$ that yields a $\delta-$PAC guarantee for any BAI algorithm, irrespective of the sampling rule.
\begin{theorem}[$\delta-$PAC]
\label{theorem: delta-PAC}
    The stopping rule in~(\ref{eq:stop}) with the choice of the threshold
    \begin{align}
        \beta_t(\delta)\;\triangleq\; W_t(\varepsilon) + t\varepsilon_t + 2\log\frac{|\Theta|}{\sqrt{2\pi}} + \log\frac{t(K-1)}{2\delta}\ ,
        % \alpha\log t + \log(K-1) -\log\delta \nonumber\\&\qquad\qquad+\log(1+(\alpha-1)^{-1})\ ,
    \end{align}
    along with any arm selection strategy and the decision rule $\hat A_\tau \triangleq a_{\tau}^{\sf top}$ is $\delta-$PAC for any $\varepsilon\in\R_+$, where $|\Theta|$ denotes the volume of the space of parameters $\Theta$.
\end{theorem}
\begin{proof}
    See Appendix~\ref{proof: delta-PAC}.
\end{proof}
\noindent The stopping threshold specified in Theorem~\ref{theorem: delta-PAC} holds for a general class of bandit instances that satisfy Assumptions 1-8. We observe that the threshold depends on the volume of the parameter space $\Theta$, which is uncommon in BAI. However, despite the additional penalty of the order of $O(\log|\Theta|)$, our stopping threshold applies a very general class of parameterized bandits for which stopping thresholds do not exist for the test statistic under consideration. Even though we may use the non-parametric GLLR statistic proposed in~\cite{pmlr-v117-agrawal20a,jourdan2022}, it is computationally expensive as it requires solving a convex optimization problem in each iteration to compute the test statistic. Naturally, the statistic in~\cite{pmlr-v117-agrawal20a}, designed for non-parametric bandits, does not use the knowledge of the parametric form of the likelihood functions, which it needs to estimate from the rewards. Furthermore, when we specialize to more structured classes of bandit instances, specifically the single parameter exponential family, we can use the tighter thresholds from~\cite{Kaufmann_JMLR}. The key difficulty in the proof of Theorem~\ref{theorem: delta-PAC} for the general case is that the log-likelihood function is generally non-linear in the reward. This is in contrast to the exponential family, which has a linear log-likelihood function in terms of the reward, which facilitates an intelligent mixture martingale construction for designing the stopping threshold~\cite{Kaufmann_JMLR}.

Next, we state the results related to the asymptotic optimality of TCB and ITCB in terms of the average sample complexity. We begin by establishing a few properties of the problem complexity that are useful for characterizing the sample complexity. Specifically, the form of the problem complexity $\Gamma(\bnu)$ in Lemma~\ref{lemma:simplified problem complexity} is instrumental in characterizing the key properties of $\Gamma(\bnu)$ and establishing an upper bound on the average sample complexity of TCB and ITCB. The following lemma characterizes these properties, which include the continuity of the problem complexity in terms of the bandit instance $\bnu\in\mcM$ in the metric space $(\mcM,D_{\sf TV})$ and the characterization of an optimal allocation of samples that maximizes the problem complexity. 
\begin{lemma}[Properties of $\Gamma(\bnu)$] \label{lemma:properties of problem complexity}
The problem complexity $\Gamma(\bnu)$ has the following properties:
    \begin{enumerate}
        \item Functions $d_{\sf U}(\cdot,\cdot)$ and $d_{\sf L}(\cdot,\cdot)$ are strictly convex in their second arguments.
        \item Problem complexity $\Gamma : \mcM\mapsto\R$ and the optimal allocation $\bw : \mcM\mapsto\Delta^K$ are continuous functions on the metric space $(\mcM,D_{\sf TV})$. Furthermore, an optimal sampling proportion is given by the unique allocation that satisfies:
        \begin{align}
            \Gamma_i(\bnu,\bw)\;=\Gamma_j(\bnu,\bw)\ , \quad \forall i,j\neq a^\star\ ,
        \end{align}
        where, for any $i\in[K]\setminus\{a^\star\}$, we have defined
        \begin{align}
        \label{eq: Gamma_i}
    \Gamma_i(\bnu,\bw)&\triangleq \inf\limits_{\bar\P\in\mcM : m(\bar\P_i)\geq m(\bar\P_{a^\star})}\;\Big \{ w_{a^\star}D_{\sf KL}(\P_{a^\star}\|\bar\P_{a^\star}) + w_iD_{\sf KL}(\P_i\|\bar\P_i)\Big \}\ .
\end{align}
    \end{enumerate}
\end{lemma}
\begin{proof}
    See Appendix~\ref{proof: properties of problem complexity}.
\end{proof}
\noindent The continuity of $\Gamma$ has been previously established by~\cite{pmlr-v117-agrawal20a} in the metric space $(\mcM,W_1)$, where $W_1$ denotes the Wasserstein distance metric. However, the continuity was established under the assumption that for any $\P\in\mcQ(\Omega)$, the boundedness assumption $\E_{\P}[f(|X|)]<B$ holds for any convex and differentiable function $f$. The key difference in the proof of Lemma~\ref{lemma:properties of problem complexity} with that of~\cite[Lemma 4]{pmlr-v117-agrawal20a} is that we do not impose the boundedness assumption. Rather, we leverage the properties of the Jensen-Shannon divergence (details in Appendix~\ref{proof: properties of problem complexity}) to establish the continuity property of the problem complexity.

\noindent Leveraging the continuity property of the problem complexity proved in Lemma~\ref{lemma:properties of problem complexity}, next, we establish the convergence in the sampling proportions due to the TCB and ITCB sampling rules to that of the optimal proportions. {This is a key theoretical contribution of the paper, which helps establish the asymptotic optimality of the algorithms, compared to the $\beta-$optimality achieved by top-two algorithms. Specifically, the top-two algorithms {\em enforce} the almost sure convergence of allocation for the best arm to $\beta$, which accordingly ensures the convergence to the $\beta-$optimal allocation. However, such an analysis does not readily extend to that of TCB and ITCB, since we do not enforce the almost sure convergence in allocation of the best arm to $\beta$. This makes the analysis of the TCB and ITCB sampling rules significantly more challenging. For details, we refer to Appendix~\ref{proof: convergence in allocation}. }
\begin{theorem}[Convergence in sampling proportions]
\label{theorem: convergence in allocation}
    If the TCB arm selection rule in~(\ref{eq: sampling rule}) is allowed to continue sampling without stopping, for any $\epsilon>0$, there exists a stochastic time instant $N_{\bw}^\epsilon$ such that 
    \begin{align}
        \left\lvert\frac{T_t(i)}{t} - w_i(\bnu)\right\rvert < \epsilon\ , \quad \forall i\in[K],\;\forall t\geq N_{\bw}^\epsilon\ .
    \end{align}
    Furthermore, $\E_{\bnu}[N_{\bw}^\epsilon]<+\infty$.
\end{theorem}
\begin{proof}
    See Appendix~\ref{proof: convergence in allocation}.
\end{proof}
\begin{theorem}[Convergence in sampling proportions]
\label{theorem: convergence in allocation 2}
    If the ITCB arm selection rule in~(\ref{eq: sampling rule 2}) is allowed to continue sampling without stopping, for any $\epsilon>0$, there exists a stochastic time instant $N_{\bw}^\epsilon$ such that 
    \begin{align}
        \left\lvert\frac{T_t(i)}{t} - w_i(\bnu)\right\rvert < \epsilon\ , \quad \forall i\in[K],\;\forall t\geq N_{\bw}^\epsilon\ .
    \end{align}
    Furthermore, $\E_{\bnu}[N_{\bw}^\epsilon]<+\infty$.
\end{theorem}
\begin{proof}
    See Appendix~\ref{proof: convergence in allocation}.
\end{proof}
\noindent Finally, leveraging Theorem~\ref{theorem: convergence in allocation} (or~\ref{theorem: convergence in allocation 2}), we characterize an upper bound on the average sample complexity of the TCB and ITCB algorithms. 
\begin{theorem}[Achievable sample complexity]
\label{theorem: SC upper bound}
    The TCB and ITCB algorithms, comprising the arm selection rules in~(\ref{eq: sampling rule}) and~(\ref{eq: sampling rule 2}) and the stopping rule in~(\ref{eq:stop}), satisfy the following upper-bound on the average sample complexity.
    \begin{align}
        \limsup\limits_{\delta\rightarrow 0}\; \frac{\E_{\bnu}[\tau]}{\log(1/\delta)}\leq \frac{1+\alpha}{\Gamma(\bnu)}\ ,
        \label{eq:the_SC_UB}
    \end{align}
    for any $\alpha>0$.
\end{theorem}
\begin{proof}
    See Appendix~\ref{proof: SC upper bound}.
\end{proof}
\noindent The upper bound in Theorem~\ref{theorem: SC upper bound} matches the information-theoretic lower bound on the average sample complexity provided by~\cite{pmlr-v49-garivier16a} up to any $\alpha>0$. This establishes the asymptotic optimality of the TCB and ITCB algorithms. Note that in the special case that the bandit instance belongs to the single parameter exponential family, we can use the tighter stopping threshold from~\cite{Kaufmann_JMLR}, in which case we can tighten the sample complexity bound in Theorem~\ref{theorem: SC upper bound}. Specifically, we may achieve a sample complexity bound with $\alpha=0$ in~\eqref{eq:the_SC_UB}. The proof follows similar arguments as~\cite[Theorem 5]{mukherjee2022}.

\section{Numerical Experiments}
In this section, we provide numerical evaluations of the different aspects of processes and decisions involved in the TCB and ITCB algorithms. First, we evaluate the convergence of the TCB and ITCB sampling rules in Gaussian and Bernoulli bandit instances. These bandit instances are summarized in Table~\ref{table:convergence}. We have selected two slippage bandit instances (i.e., instances with identical arm means for the sub-optimal arms) and two instances having distinct arm means. All experiments are averaged over $10^4$ independent trials. We show convergence in the following three senses.

\begin{table}[h]
\centering
\begin{tabular}{@{}|c|c|c|@{}}
 \hline  
Distribution & $\bnu$ & $\bw(\bnu)$
\\ \hline \hline
Gaussian  & $\bnu_1 = [100, 2, 1]$  & $[0.4143,0.2979,0.2878]$                                 \\  \hline                                               
Gaussian  & $\bnu_2 = [1.2, 1, 1, 1, 1]$                & $[1/3, 1/6, 1/6, 1/6, 1/6]$
\\ \hline 
Bernoulli & $\bnu_3= [0.8, 0.45, 0.45, 0.45]$            & $[0.3854, 0.2049, 0.2049, 0.2049]$
\\  \hline
Bernoulli & $\bnu_4 = [0.79, 0.29, 0.289]$ & 
$[0.426, 0.2880, 0.2860]$                                \\  \hline
\end{tabular}
\caption{Bandit instances and their associated sampling weights.}
\label{table:convergence}
\end{table}

\begin{figure}[h]
    \centering
   \begin{subfigure}{0.45\textwidth}
         \centering
         \includegraphics[width=.95\textwidth]{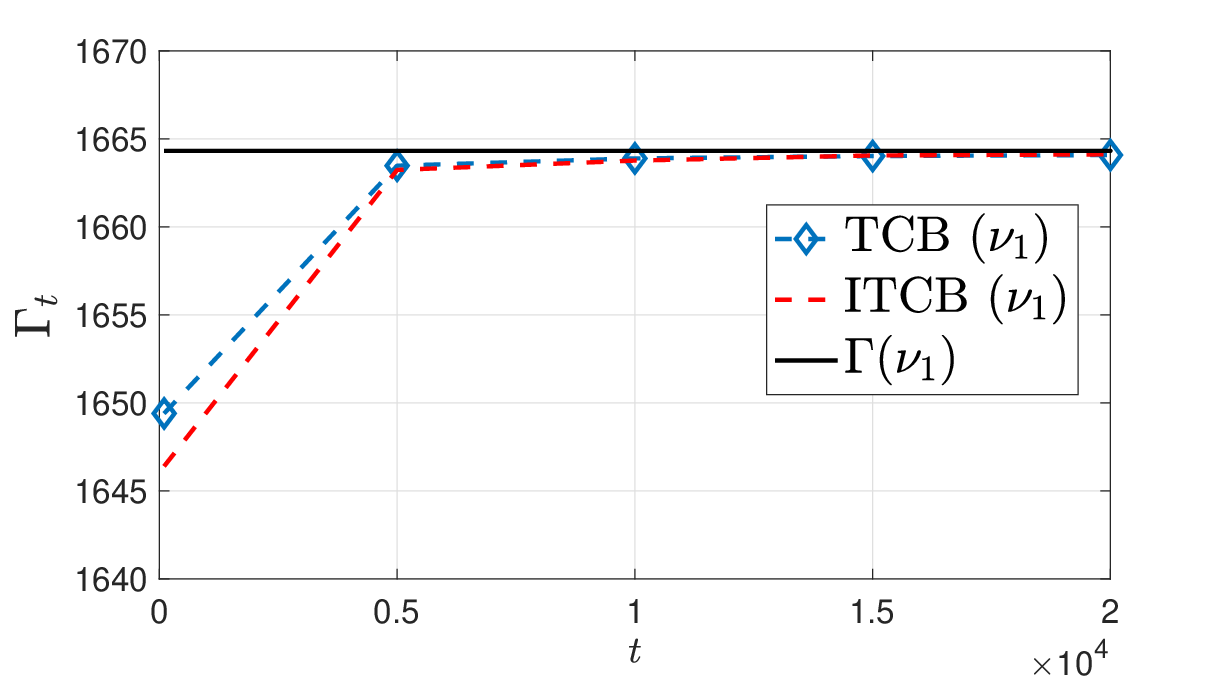}
        \caption{bandit instance $\bnu_1$}
         \label{fig:tc_v1}
     \end{subfigure} 
     \begin{subfigure}{0.45\textwidth}
         \centering
         \includegraphics[width=.95\textwidth]{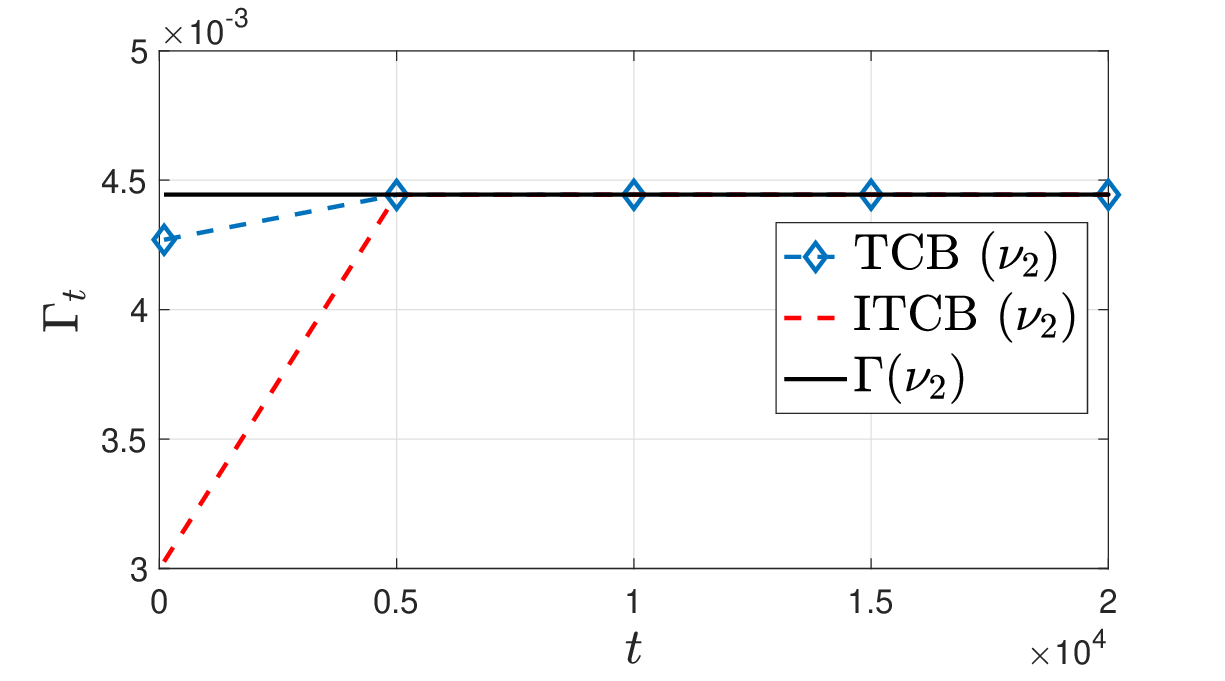}
        \caption{bandit instance $\bnu_2$}
         \label{fig:tc_v2}
     \end{subfigure}
     \begin{subfigure}{0.45\textwidth}
         \centering
         \includegraphics[width=.95\textwidth]{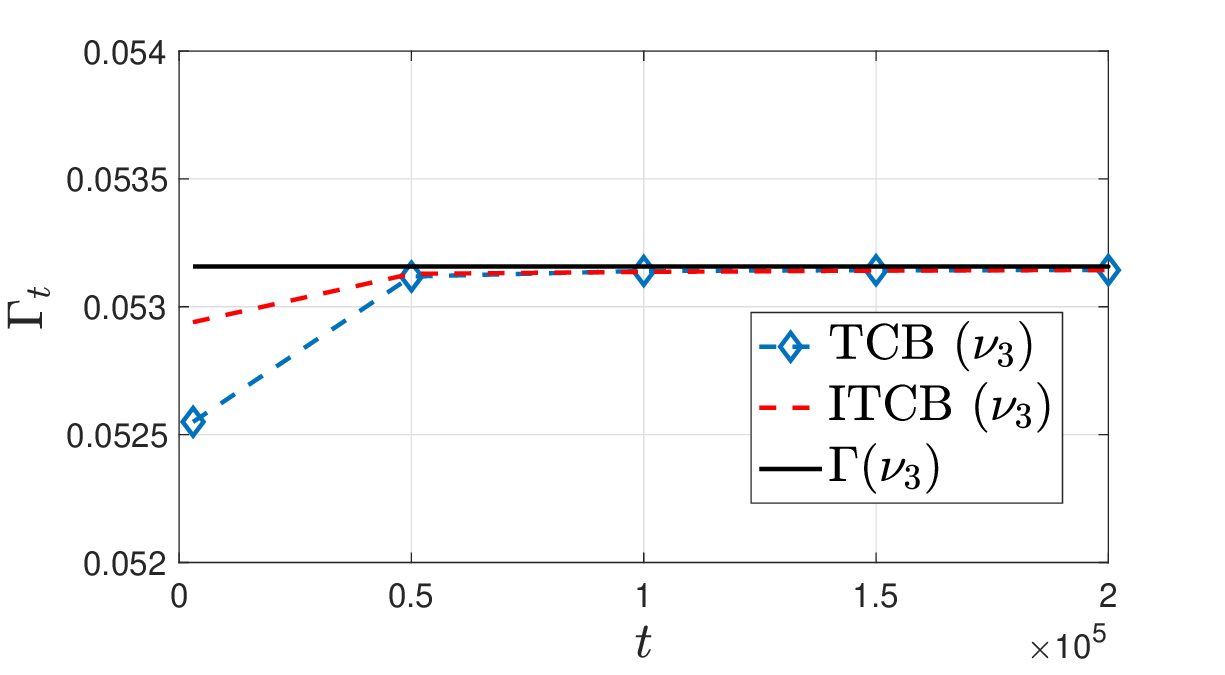}
        \caption{bandit instance $\bnu_3$}
         \label{fig:tc_v3}
     \end{subfigure}
     \begin{subfigure}{0.45\textwidth}
         \centering
         \includegraphics[width=.95\textwidth]{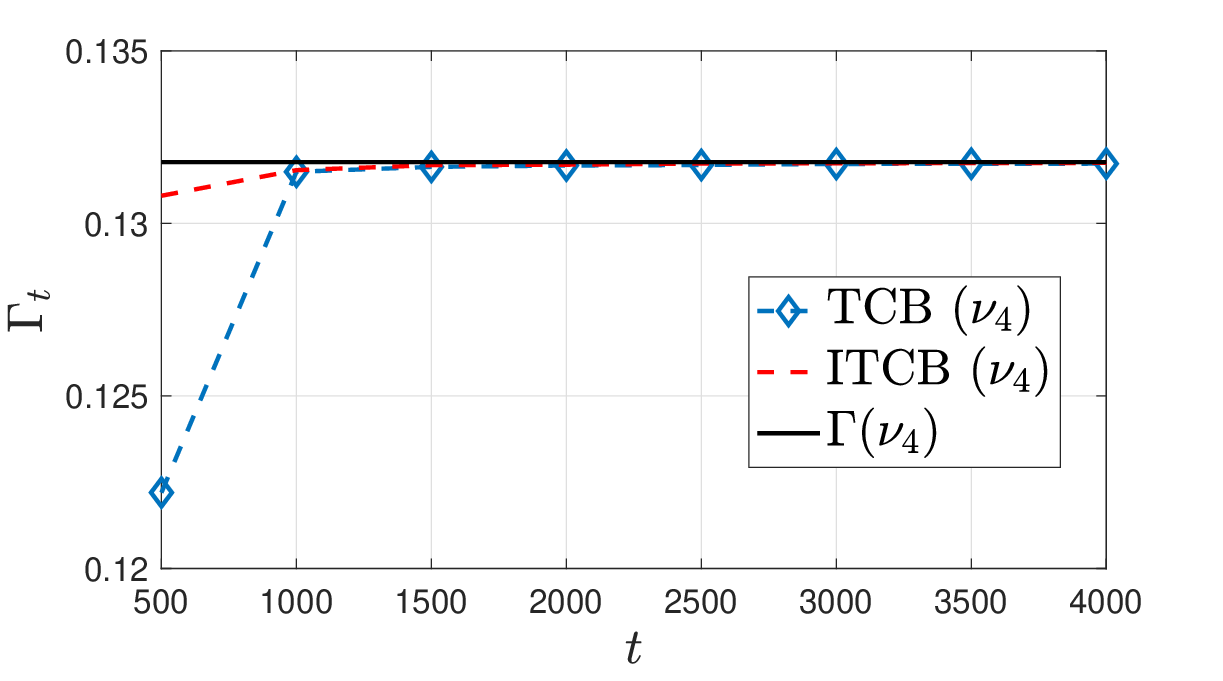}
        \caption{bandit instance $\bnu_4$}
         \label{fig:tc_v4}
     \end{subfigure}
    \caption{Variations of the estimates of the transportation cost $\Gamma_t$ over time.}
    \label{fig:tc}
\end{figure}

\paragraph{Transportation cost estimates.} The key idea of the TCB and ITCB arm selection rules is to estimate the problem complexity $\Gamma(\bnu)$. This is achieved by maximizing the transportation cost $\Gamma_t\left(\frac{1}{t}\bT_t\right)$, which acts as an estimate of the problem complexity. The analysis of the average sample complexity of the TCB and ITCB algorithms provided in Theorem~\ref{theorem: SC upper bound} critically hinges on the convergence in the transportation cost to the problem complexity (see Appendix~\ref{proof: SC upper bound} for details). To empirically evaluate this, in figures~\ref{fig:tc_v1}-\ref{fig:tc_v4} we illustrate the term $\Gamma_t\left(\frac{1}{t}\bT_t\right)$ for the problem instances specified in Table~\ref{table:convergence}. These empirical results show the convergence of the transportation costs to their optimal values (problem complexity). We note that the range of the transportation cost depends on the mean values of the arms, and consequently, the problem complexity of instance $\bnu_1$ (with $\mu(1) = 100$) is significantly larger than the other bandit instances whose mean values are closer to $1$.

\paragraph{Estimates of $\beta$.} An optimal sampling rule design for $\delta-$PAC BAI is a functional estimation problem in which the learner forms estimates of the problem complexity for arm selection decisions in the form of transportation costs. These transportation costs, among other parameters, are also functions of the sampling proportion assigned to the best arm $a^\star$. Hence, as established in Theorem~\ref{theorem: convergence in allocation} and Theorem~\ref{theorem: convergence in allocation 2}, the TCB and ITCB sampling strategies implicitly {\em estimate} $\beta$, i.e., the optimal sampling proportion for the best arm. To empirically establish the variations of the estimates of the allocation over time, in figures~\ref{fig:beta_v1}-\ref{fig:beta_v4}, we illustrate the allocation of sampling resources to the best arm and compare it to the optimal choice of $\beta=w_{a^\star}$ in the four instances specified in Table~\ref{table:convergence}. Similarly to the transportation cost estimates, it is observed that the estimates $\beta_t$ converge to the optimal value.

\begin{figure}[h]
    \centering
   \begin{subfigure}{0.45\textwidth}
         \centering
         \includegraphics[width=.95\textwidth]{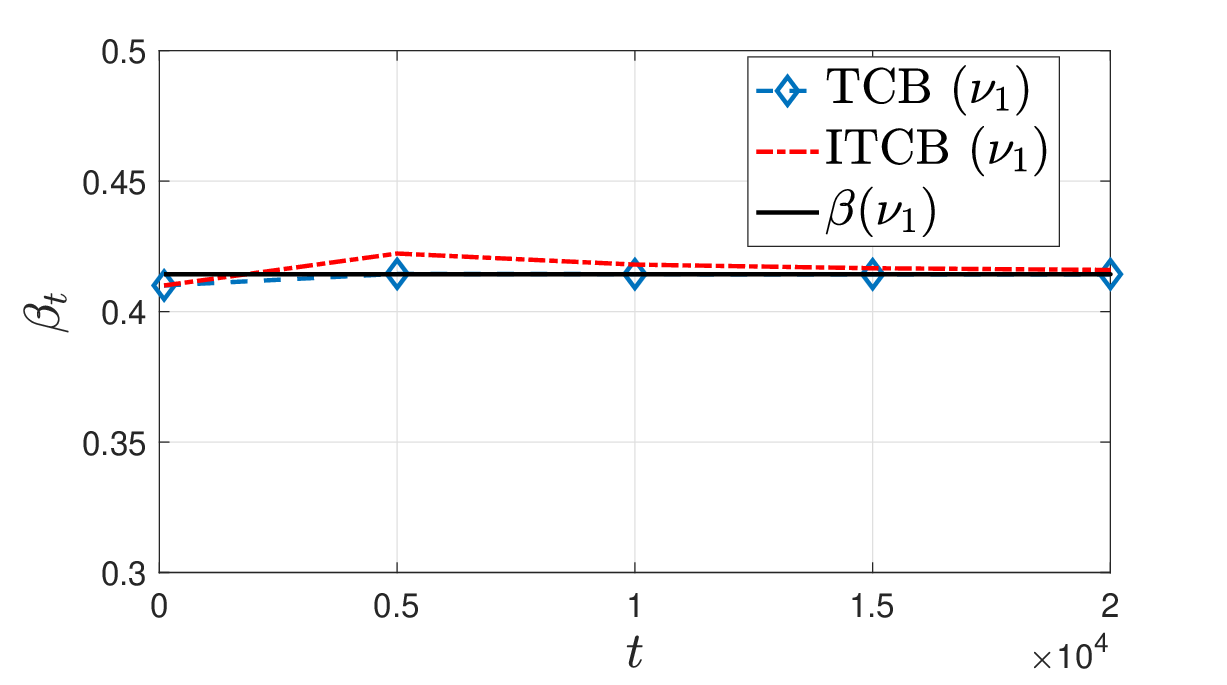}
        \caption{bandit instance $\bnu_1$}
         \label{fig:beta_v1}
     \end{subfigure} 
     \begin{subfigure}{0.45\textwidth}
         \centering
         \includegraphics[width=.95\textwidth]{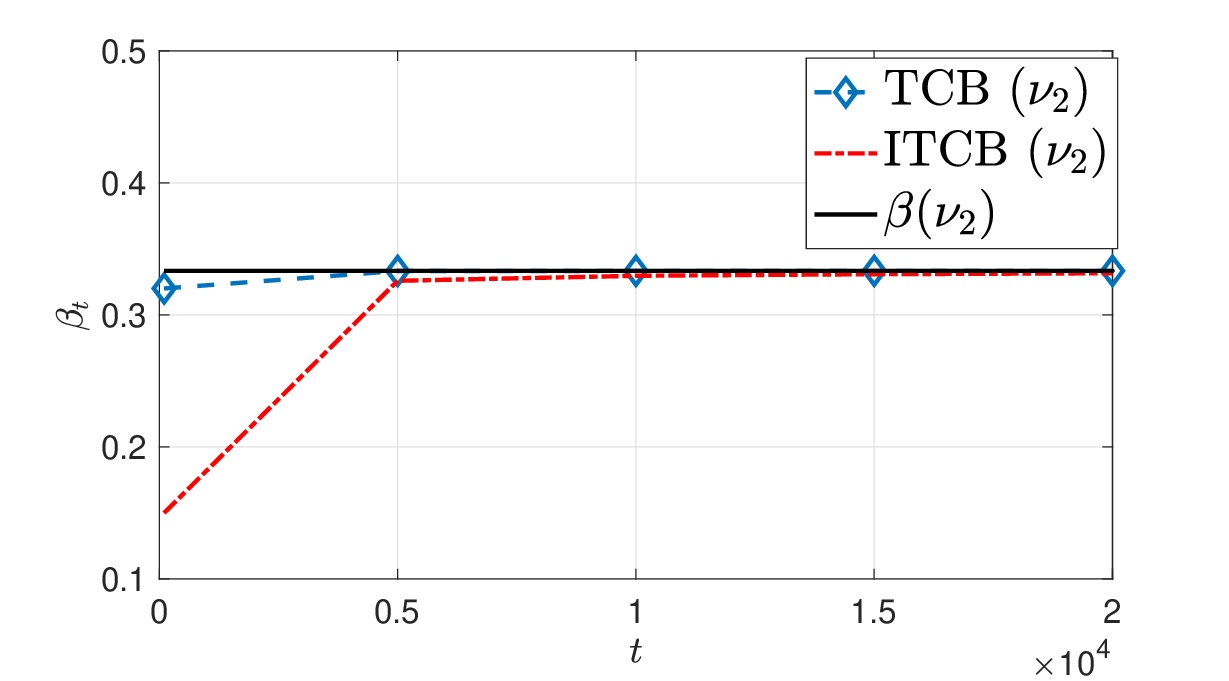}
        \caption{bandit instance $\bnu_2$}
         \label{fig:beta_v2}
     \end{subfigure}
     \begin{subfigure}{0.45\textwidth}
         \centering
         \includegraphics[width=.95\textwidth]{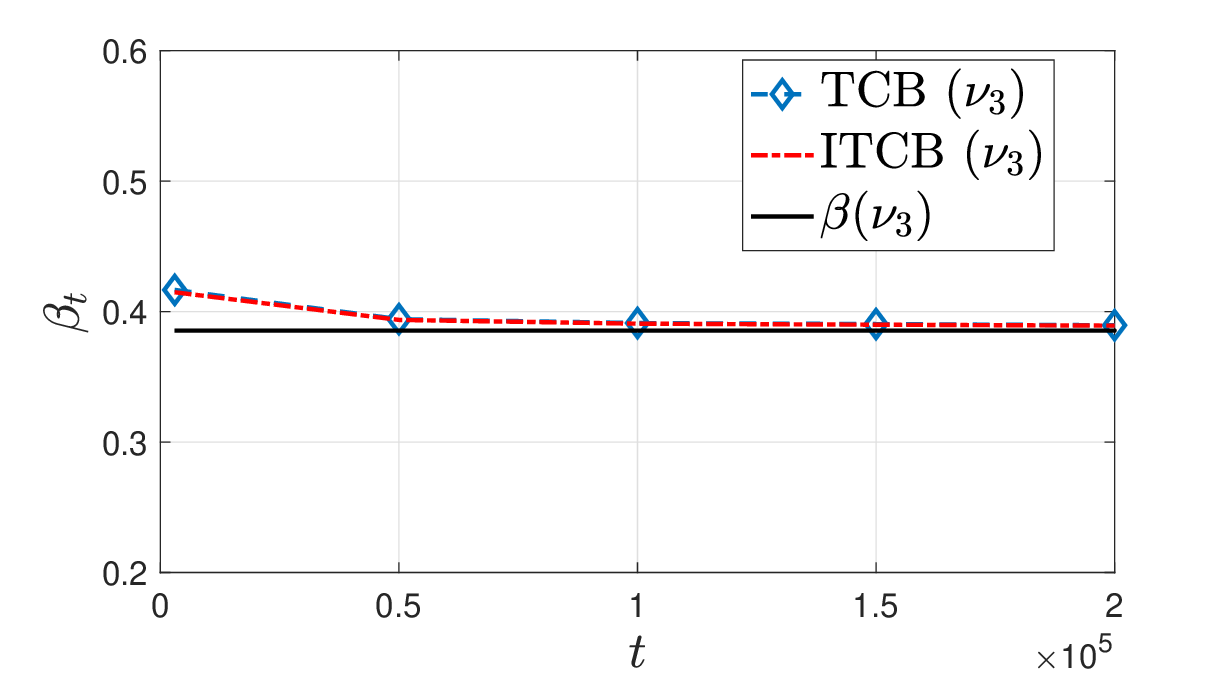}
        \caption{bandit instance $\bnu_3$}
         \label{fig:beta_v3}
     \end{subfigure}
     \begin{subfigure}{0.45\textwidth}
         \centering
         \includegraphics[width=.95\textwidth]{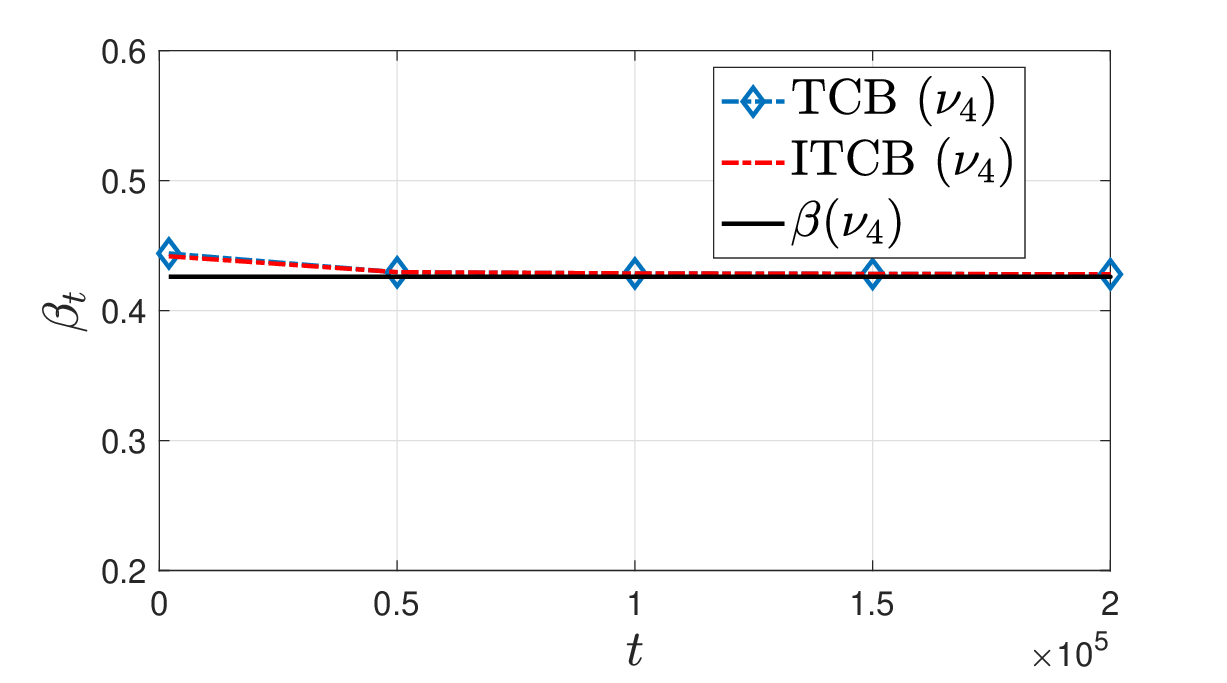}
        \caption{bandit instance $\bnu_4$}
         \label{fig:beta_v4}
     \end{subfigure}
    \caption{Variations of the estimates $\beta_t$ over time.}
    \label{fig:beta}
\end{figure}

\paragraph{Sampling over-sampled arms.} The analysis of the TCB and ITCB algorithms show that their sampling rules eventually sample from the set of {\em under-sampled} arms, i.e., the set of arms that have been sampled fewer times compared to the optimal allocation (for details, see Appendix~\ref{proof: convergence in allocation}). We demonstrate that this is a key property of the TCB and ITCB sampling rules that enables the convergence of the transportation cost to the problem complexity. To show this, we devise an algorithm that preserves the explicit exploration phase of the TCB and ITCB algorithms for convergence in the arm means. However, if the set of under-explored arms is empty, this algorithm always selects the current best arm, i.e., $a_t^{\sf top}$. It can be readily verified that due to the convergence in mean for all the arms, $a_t^{\sf top}$ converges to the best arm $a^\star$.
Furthermore, the best arm $a^\star$ eventually gets over-sampled since we put the entire allocation on $a^\star$. Figure~\ref{fig:bestarm} shows the deviation of the transportation cost from the problem complexity over time. We observe that the transportation cost diverges from the problem complexity, establishing that sampling from the set of over-sampled arms cannot be optimal.

\begin{figure}[h]
    \centering
    \includegraphics[width=0.5\textwidth]{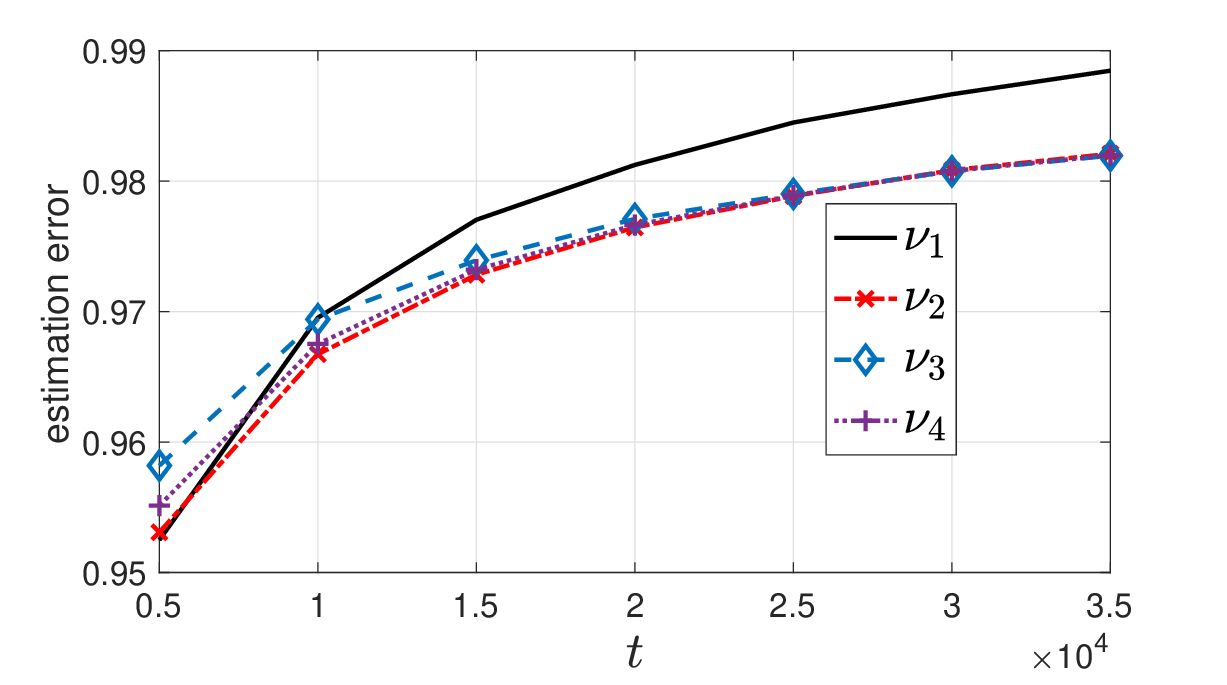}
    \caption{Divergence of transportation cost without controlling under-sampling.}
    \label{fig:bestarm}
\end{figure}

\paragraph{Sensitivity to $\beta$.} Next, we assess the dependency of the state-of-the-art BAI algorithms on $\beta$, and compare the empirical performance of the TCB and ITCB algorithms to these algorithms. These state-of-the-art algorithms for BAI include: T3C~\cite{pmlr-v108-shang20a}, TT-SPRT~\cite{mukherjee2022,mukherjee_SPRT_conf}, TS-TCI~\cite{jourdan2022}, EB-TCI~\cite{jourdan2022}, and FW~\cite{FW}. TT-SPRT and EB-TCI use the empirical best arm as the leader. TT-SPRT uses the arm with the closest GLLR statistic to the empirical best arm as the challenger, and EB-TCI includes an additional exploration penalty to the GLLR statistic to determine the challenger. Distinct from TT-SPRT and EB-TCI, T3C and TS-TCI use Thompson sampling from the posterior distribution to identify the leader. T3C selects the challenger as the arm with the closest GLLR statistic to the leader, while TS-TCI uses a penalized GLLR statistic to promote exploration. FW selects the next arm based on a Frank-Wolfe-based update step to solve~(\ref{eq: optimal_proportions}) based on the current mean estimates.
Furthermore, we also compare the T3C algorithm with the $\beta$-tuning routine from~\cite{russo2016}. Note that we have not compared TTTS with $\beta$-tuning, owing to its large computation time required in identifying the challenger. We perform our experiments based on two common reward distributions, Bernoulli and Gaussian bandits. All the experiments are averaged over $2000$ independent Monte Carlo trials, and we have set $\delta = 10^{-8}$. We have the following two sets of experiments. 
\vspace{0.05in}
\begin{figure}[h]
    \centering
    \begin{minipage}{0.49\textwidth}
        \centering
        \includegraphics[width=0.99\textwidth]{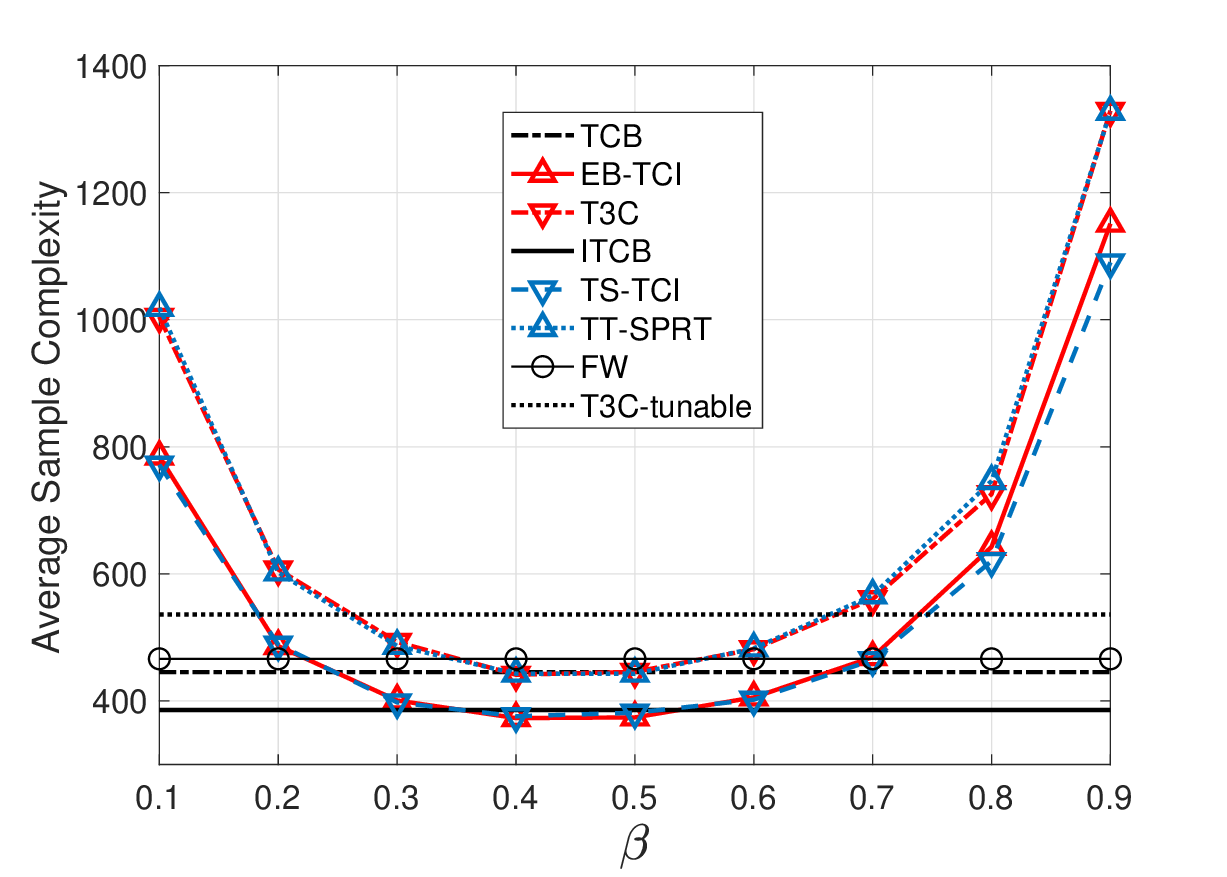} % first figure itself
        \caption{Sensitivity to $\beta$ in the Bernoulli instance $[0.8,0.5,0.3,0.29,0.06]$ with $\delta=10^{-8}$.}
        \label{fig:3}
    \end{minipage}\hfill
    \begin{minipage}{0.49\textwidth}
        \centering
        \includegraphics[width=0.99\textwidth]{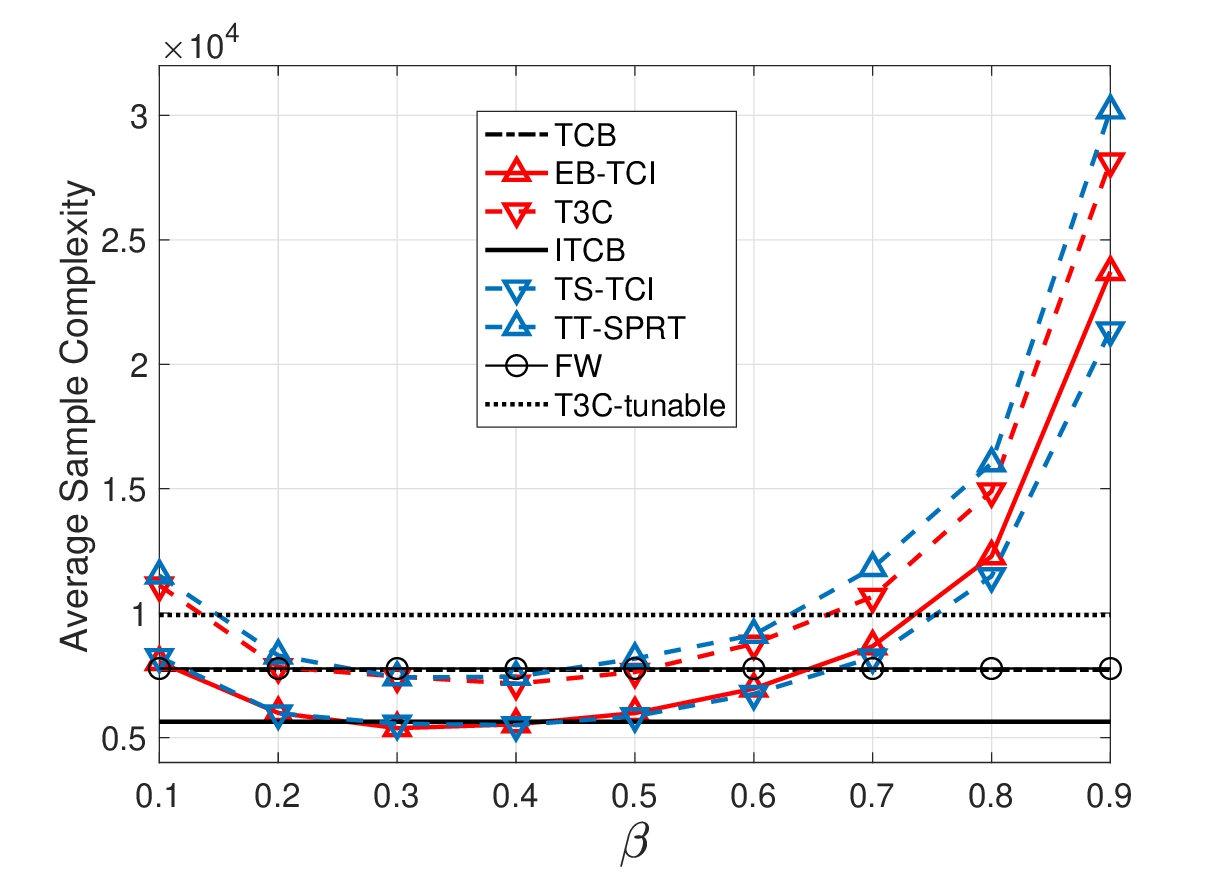} % second figure itself
        \caption{Sensitivity to $\beta$ in the Gaussian instance $[1.2,1,1,1,1]$ with $\delta=10^{-8}$.}
        \label{fig:4}
    \end{minipage}
    \end{figure}

\begin{enumerate}
    \item  \emph{Bernoulli.} In this experiment, we use a Bernoulli bandit with mean values $[0.8,0.5,0.3,0.29,0.06]$ and set $\delta\triangleq 10^{-8}$. Figure~\ref{fig:3} shows the performance of the TCB and ITCB arm selection rules compared to the state-of-the-art. We observe that the TCB and ITCB arm selection strategies in~(\ref{eq: sampling rule}) and~(\ref{eq: sampling rule 2}) are agnostic to the parameter $\beta$. Furthermore, ITCB outperforms the top-two sampling strategies for various values of the tuning parameter $\beta$, and its performance is comparable to the optimization-based sampling rule FW. Furthermore, the performance of ITCB matches that of TS-TCI and EB-TCI at $\beta\approx 0.4$. 

    \item \emph{Gaussian.} For the next experiment, we take a Gaussian bandit instance with mean values given by $[1.2,1,1,1,1]$. Figure~\ref{fig:4} compares the TCB and ITCB algorithms against state-of-the-art BAI algorithms for various values of $\beta$. We set $\delta\triangleq 10^{-8}$ for this experiment. We observe that the ITCB algorithm outperforms the top-two algorithms in various regimes of $\beta$, and the performances of TS-TCI and EB-TCI match that of ITCB in the range $0.3\leq\beta\leq 0.4$. 
\end{enumerate}
\noindent

\section{Conclusions} 
We have investigated the problem of fixed-confidence best-arm identification (BAI) in stochastic multi-arm bandits (MABs). We have designed the transport cost balancing (TCB) algorithm, in which the key decision is finding the optimal allocation of the sampling resources among different arms. This extends the existing strategies that aim to allocate a $\beta$ fraction of the resources to the best arm and achieve $\beta-$optimality, with $\beta$ remaining a parameter for the sampling routine. In the proposed TCB algorithm, the optimal value of $\beta$ is also {\em implicitly} estimated, rendering the algorithm independent of $\beta$.  As a result, we have established that the proposed TCB algorithm is asymptotically optimal. We have also extended the TCB algorithm by including an additional exploration penalty based on the number of times each arm is chosen. This algorithm, referred to as improved TCB (ITCB), is also shown to achieve asymptotic optimality and improved empirical performance compared to TCB.

\appendices

\newpage 

\section{Proof of Theorem~\ref{theorem: delta-PAC}}
\label{proof: delta-PAC}
The proof of Theorem~\ref{theorem: delta-PAC} is based on upper bounding the GLLR statistic using a mixture martingale construction, which we state in Lemma~\ref{lemma:KL_UB}. Subsequently, we use Ville's supermartingale inequality to upper bound the probability of an incorrect terminal decision. The martingale construction closely resembles Laplace's approximation method for approximating the maximum value of any function~\cite{bandit} and is based on an upper bound on the maximum of a sum of twice-differentiable functions, which is stated in Lemma~\ref{lemma:Taylor}.

\begin{lemma}
\label{lemma:Taylor}
    For any sequence $\{g_s : s\in[t]\}$ of twice-differentiable functions $g : \Theta \mapsto \R$, where $\Theta$ is a compact space, let us define the maximizer
    \begin{align}
        \mu_t\;\triangleq\; \argmax\limits_{\rho\in\Theta}\; \sum\limits_{s\in[t]} g_s(\rho)\ .
    \end{align}
    Furthermore, let us denote $\eta$ as the uniform distribution over $\Theta$. Then, for any $\varepsilon\in\R_+$, we have
    \begin{align}
        \sum\limits_{s\in[t]} g_s(\mu_t)\;&\leq\; \log\E_{\eta}\left[ \exp\left ( \sum\limits_{s\in[t]} g_s(\rho)\right)\right] + \frac{1}{2}\log V_t -\log\left( 1 - 2Q(\varepsilon\sqrt{V_t})\right)\nonumber\\
        &\qquad- \min\limits_{\rho\in[\mu_t-\varepsilon,\mu_t+\varepsilon]} \left\{\sum\limits_{s\in[t]} \left ( g_s(\rho)-g_s(\mu_t)\right )\right\} + \log\frac{|\Theta|}{\sqrt{2\pi}}\ ,
        \label{eq:th1_lemma1}
    \end{align}
    where we have defined 
    \begin{align}
        V_t\;\triangleq\; - \sum\limits_{s\in[t]} g^{\prime\prime}_s(\rho)\bigg\lvert_{\rho=\mu_t}\ .
    \end{align}
\end{lemma}
\begin{proof}
    Using Taylor's expansion, we obtain
    \begin{align}
        \sum\limits_{s\in[t]} g_s(\rho)\;=\; \sum\limits_{s\in[t]} g_s(\mu_t) - \frac{1}{2} (\rho-\mu_t)^2V_t + R(\rho,\mu_t)\ ,
        \label{eq:th1_Taylor1}
    \end{align}
    where we have defined\footnote{Note that we only assume $g$ to be twice-differentiable. In case the higher order derivatives don't exist, we may simply define~$R(\rho,\mu_t)\triangleq \sum_{s\in[t]} (g_s(\rho)-g_s(\mu_t)) + \frac{1}{2}(\rho-\mu_t)^2V_t$.} 
    \begin{align}
        R(\rho,\mu_t)\;\triangleq\; \sum\limits_{s\in[t]}\sum\limits_{j=3}^{\infty} \frac{1}{j!} g_s^{(j)}(\rho-\mu_t)^j\ ,
    \end{align}
    and $g_s^{(j)}(\mu_t)$ denotes the $j^{\rm th}$ derivative of the function $g_s$ evaluated at $\mu_t$. Furthermore, we have 
    \begin{align}
        \displaystyle\bigintsss_{\Theta} \exp\left ( \sum\limits_{s\in[t]} g_s(\rho)\right ) \diff\eta(\rho)\;&\stackrel{\eqref{eq:th1_Taylor1}}{=}\;\frac{1}{|\Theta|} \exp\left ( \sum\limits_{s\in[t]} g_s(\mu_t)\right ) \displaystyle\bigintss_{\Theta} \exp\left ( -\frac{1}{2}(\rho-\mu_t)^2V_t\right )\cdot\exp\left ( R(\rho,\mu_t)\right )\diff \rho\\
        & \geq\;\frac{1}{|\Theta|} \exp\left ( \sum\limits_{s\in[t]} g_s(\mu_t)\right ) \displaystyle\bigintsss_{\mu_t-\varepsilon}^{\mu_t + \varepsilon} \exp\left ( -\frac{1}{2}(\rho-\mu_t)^2V_t\right )\cdot\exp\left ( R(\rho,\mu_t)\right )\diff \rho\ ,
        \label{eq:th1_2}
    \end{align}
    where the right hand side in~\eqref{eq:th1_2} is positive due to the fact that $\mu_t\in\Theta$. Furthermore, let us define
    \begin{align}
        R_t^{\min}\;\triangleq\; \min\limits_{\rho\in[\mu_t - \epsilon, \mu_t + \epsilon]} R(\rho,\mu_t)\ .
        \label{eq:th1_rtmin}
    \end{align}
    Using~\eqref{eq:th1_rtmin},~\eqref{eq:th1_2} can be lower bounded as
    \begin{align}
        \displaystyle\bigintsss_{\Theta} \exp\left( \sum\limits_{s\in[t]} g_s(\rho)\right )\diff\eta(\rho)\;\geq\; \frac{1}{|\Theta|} \exp\left ( \sum\limits_{s\in[t]} g_s(\mu_t) + R_t^{\min}\right )\displaystyle\bigintsss_{\mu_t-\varepsilon}^{\mu_t+\varepsilon} \exp\left( -\frac{1}{2}(\rho-\mu_t)^2V_t\right )\diff\rho\ .
        \label{eq:th1_3}
    \end{align}
    Let us make the following change of variables.
    \begin{align}
        y\;=\;(\rho-\mu_t)\sqrt{V_t}\ , \qquad y_t\;=\;\varepsilon\sqrt{V_t}\ .
        \label{eq:th1_4}
    \end{align}
    We have
    \begin{align}
        \displaystyle\bigintsss_{\Theta} \exp\left ( \sum\limits_{s\in[t]} g_s(\rho)\right ) \diff\eta(\rho)\;&\stackrel{\eqref{eq:th1_3}-\eqref{eq:th1_4}}{\geq}\;\frac{1}{|\Theta|}\exp\left( \sum\limits_{s\in[t]} g_s(\mu_t) + R_t^{\min}\right)\cdot\frac{1}{\sqrt{V_t}}\displaystyle\bigintsss_{-y_t}^{y_t} \exp\left( -\frac{y^2}{2}\right)\diff y\\
        &\quad=\;\frac{1}{|\Theta|} \exp\left ( \sum\limits_{s\in[t]} g_s(\mu_t) + R_t^{\min}\right )\cdot\sqrt{\frac{2\pi}{V_t}}\left(1-2Q(\varepsilon\sqrt{V_t})\right)\ ,
        \label{eq:th1_5}
    \end{align}
    where $Q(x)$ denotes the $Q$ function evaluated at any point $x\in\R$. Taking $\log$ on both sides of~\eqref{eq:th1_5} and rearranging, we have
    \begin{align}
        \sum\limits_{s\in[t]} g_s(\mu_t)\;\leq\; \log\E_{\eta}\left [ \exp\left( \sum\limits_{s\in[t]} g_s(\rho)\right)\right ] + \frac{1}{2}\log V_t -\log\left( 1-2Q(\varepsilon\sqrt{V_t})\right) - R_t^{\min} + \log\frac{|\Theta|}{\sqrt{2\pi}}\ .
        \label{eq:th1_6}
    \end{align}
    Furthermore, note that
    \begin{align}
        R_t^{\min}\;&\stackrel{\eqref{eq:th1_Taylor1}}{=}\;\min\limits_{\rho\in[\mu_t-\varepsilon,\mu_t+\varepsilon]}\;\left\{ \sum\limits_{s\in[t]} g_s(\rho) + \frac{1}{2}(\rho-\mu_t)^2V_t\right\} - \sum\limits_{s\in[t]}g_s(\mu_t)\\
        &\;\;\geq\; \min\limits_{[\mu_t-\varepsilon,\mu_t+\varepsilon]} \left\{ \sum\limits_{s\in[t]} \left( g_s(\rho) - g_s(\mu_t)\right)\right\}\ .
        \label{eq:th1_7}
    \end{align}
    Finally, combining~\eqref{eq:th1_6} and~\eqref{eq:th1_7}, we obtain~\eqref{eq:th1_lemma1}.
\end{proof}
% \log\left( 1-2Q\left(\varepsilon\sqrt{T_t(i)\mcI_i(\mu_t(i))}\right)\right)
\begin{lemma}
\label{lemma:KL_UB}
    For any arm $i\in[K]$ and for any $\varepsilon\in\R_+$, there exists a non-negative martingale $M_t(i)$ satisfying $\E[M_1(i)] = 1$, such that
    \begin{align}
        T_t(i) d_i(\mu_t(i)\|\mu(i))\;&\leq\;\log M_t(i) + \frac{1}{2}\log\left(T_t(i)\mcI_i(\mu_t(i))\right) - W_t(\varepsilon,i)+ \log\frac{|\Theta|}{\sqrt{2\pi}}\nonumber\\
        &\qquad +T_t(i)\max\left\{d_i(\mu_t(i)\|\mu_t(i)-\varepsilon), d_i(\mu_t(i)\|\mu_t(i)+\varepsilon) \right\} \ .
        \label{eq:lemma_KL_UB}
    \end{align}
\end{lemma}
\begin{proof}
    \begin{align}
        T_t(i)d_i(\mu_t(i)\|\mu(i))\;&=\;\sum\limits_{s\in[t]:A_s=i}\displaystyle\bigintsss_{\Omega}\left (\log\frac{\pi_i(X\med\mu_t(i))}{\pi_i(X\med\mu(i))}\right)\pi_i(X\med\mu_t(i))\diff X\\
        &=\;\sum\limits_{s\in[t]:A_s=i}\displaystyle\bigintsss_{\Omega} \left( \log\pi_i(X\med\mu_t(i))\right)\pi_i(X\med\mu_t(i))\diff X  \nonumber\\ &\qquad\qquad-\sum\limits_{s\in[t]:A_s=i}\displaystyle\bigintsss_{\Omega} \left( \log\pi_i(X\med\mu(i))\right)\pi_i(X\med\mu_t(i))\diff X\ .
        \label{eq:th1_8}
    \end{align}
    Next, we will leverage Lemma~\ref{lemma:Taylor} using $g(\rho) = \log \pi_i(X\med\rho)$. Note that
    \begin{align}
    \label{eq:th1_9a}
        &\displaystyle\bigintsss_{\Omega^{\otimes T_t(i)}} \Big( \log \bar V_t(i)\Big)\displaystyle\prod\limits_{s\in[t]:A_s = i} \pi_i(X_s\med\mu_t(i))\diff\mcX_t^i\nonumber\\
        &\qquad \leq\; \log\displaystyle\bigintsss_{\Omega^{\otimes T_t(i)}} \bar V_t(i)\displaystyle\prod\limits_{s\in[t]:A_s = i}\pi_i(X_s\med\mu_t(i)) \diff\mcX_t^i\\
        &\qquad=\;\log\displaystyle\bigintsss_{\Omega^{\otimes T_t(i)}} \sum\limits_{s\in[t]:A_s = i} \left ( -\frac{\partial^2}{\partial\theta^2}\log \pi_i(X_s\med\theta)\right )_{\theta=\mu_t(i)}\displaystyle\prod\limits_{s\in[t]:A_s = i} \pi_i(X_s\med \mu_t(i))\diff\mcX_t^i\\
        \label{eq:th1_9b}
        & \qquad=\; \log\sum\limits_{s\in[t]:A_s=i}\displaystyle\bigintsss_{\Omega} \left ( -\frac{\partial^2}{\partial\theta^2}\log \pi_i(X\med\theta)\right )_{\theta=\mu_t(i)}\pi_i(X\med\mu_t(i))\diff X\\
        &\qquad =\;\sum\limits_{s\in[t]:A_s=i} \mcI_i(\mu_t(i))\\
        &\qquad =\;T_t(i)\mcI_i(\mu_t(i))\ ,
        \label{eq:th1_9}
    \end{align}
    where~\eqref{eq:th1_9a} is obtained using Jensen's inequality, and~\eqref{eq:th1_9b} is a result of applying the Fubini-Tonelli's theorem.
    Furthermore, we have
    \begin{align}
        \label{eq:th1_9c}
        &-\displaystyle\bigintsss_{\Omega^{\otimes T_t(i)}} \min\limits_{\rho\in[\mu_t(i)-\varepsilon,\mu_t(i)+\varepsilon]} \sum\limits_{s\in[t]:A_s = i} \log\frac{\pi_i(X_s\med\rho)}{\pi_i(X_s\med\mu_t(i))}\displaystyle\prod\limits_{s\in[t]:A_s = i} \pi_i(X_s\med\mu_t(i))\diff\mcX_t^i\nonumber\\
        &\quad\leq - \min\limits_{\rho\in[\mu_t(i)-\varepsilon,\mu_t(i)+\varepsilon]} \displaystyle\bigintsss_{\Omega^{\otimes T_t(i)}}\sum\limits_{s\in[t]:A_s = i} \log\frac{\pi_i(X_s\med\rho)}{\pi_i(X_s\med\mu_t(i))}\displaystyle\prod\limits_{s\in[t]:A_s = i} \pi_i(X_s\med\mu_t(i))\diff\mcX_t^i\\
        &\quad = - \min\limits_{s\in[t]:A_s = i} d_i(\rho\|\mu_t(i))\\
        &\quad = \max\limits_{s\in[t]:A_s = i} d_i(\mu_t(i)\|\rho)\ ,
        \label{eq:th1_9d}
    \end{align}
    where~\eqref{eq:th1_9c} follows from Jensen's inequality along with Assumption $2$. Next, using Lemma~\ref{lemma:Taylor} along with~\eqref{eq:th1_9} and~\eqref{eq:th1_9d}, \eqref{eq:th1_8} can be upper bounded by
    \begin{align}
        T_t(i)d_i(\mu_t(i)\|\mu(i))\;&\leq\;\displaystyle\bigintssss_{\Omega^{\otimes T_t(i)}}\log \E_{\eta}\left [ \exp\left( \sum\limits_{s\in[t]:A_s = i} \log\pi_i(X_s\med\rho)\right)\right]\displaystyle\prod\limits_{s\in[t]:A_s = i} \pi_i(X_s\med\mu_t(i))\diff\mcX_t^i \nonumber\\
        &\qquad -  \displaystyle\bigintssss_{\Omega^{\otimes T_t(i)}} \sum\limits_{s\in[t]:A_s=i}\left( \log \pi_i(X\med\mu(i))\right)\displaystyle\prod\limits_{s\in[t]:A_s = i}\pi_i(X_s\med\mu_t(i))\diff\mcX_t^i\nonumber\\
        &\qquad+\frac{1}{2}\log\left(T_t(i)\mcI_i(\mu_t(i))\right) - W_t(\varepsilon,i) + \log\frac{|\Theta|}{\sqrt{2\pi}}\nonumber\\
        &\qquad + T_t(i)\cdot\max\limits_{\rho\in[\mu_t(i)-\varepsilon,\mu_t(i)+\varepsilon]}d_i(\mu_t(i)\|\rho)\\
        &\leq\; \log\underbrace{\bigintsss_{\Omega^{\otimes T_t(i)}} \E_{\eta}\left[ \exp\left( \sum\limits_{s\in[t]:A_s = i} \log\frac{\pi_i(X_s\med\rho)}{\pi_i(X_s\med\mu(i))}\right)\right]\displaystyle\prod\limits_{s\in[t]:A_s = i} \pi_i(X_s\med\mu_t(i))\diff\mcX_t^i}_{\triangleq M_t(i)}\nonumber\\
        &\qquad+\frac{1}{2}\log\left(T_t(i)\mcI_i(\mu_t(i))\right) - W_t(\varepsilon,i) + \log\frac{|\Theta|}{\sqrt{2\pi}}\nonumber\\
        &\qquad + T_t(i)\cdot\max\limits_{\rho\in[\mu_t(i)-\varepsilon,\mu_t(i)+\varepsilon]}d_i(\mu_t(i)\|\rho)\ ,
        \label{eq:th1_10}
    \end{align}
    where~\eqref{eq:th1_10} is obtained by using the Jensen's inequality, and $M_t(i)$ is defined such that $M_1(i) = 1$, if $A_1\neq i$.
    
    Furthermore, let us define
    \begin{align}
        \rho_t\;\triangleq\;\argmax\limits_{\rho\in[\mu_t(i)-\varepsilon,\mu_t(i)+\varepsilon]} \left\{ T_t(i) d_i(\mu_t(i)\|\rho)\right\}\ .
    \end{align}
    There exists $\gamma\in(0,1)$ such that we have $\rho_t = \gamma(\mu_t(i) - \varepsilon) + (1-\gamma)(\mu_t(i)+\varepsilon)$. Owing to the convexity of KL divergence, we have
    \begin{align}
        d_i(\mu_t(i)\|\rho_t)\;&\leq\; \gamma d_i(\mu_t(i)\|\mu_t(i)-\varepsilon) + (1-\gamma) d_i(\mu_t(i)\|\mu_t(i)+\varepsilon)\\
        &\leq\max\{d_i(\mu_t(i)\|\mu_t(i)-\varepsilon), d_i(\mu_t(i)\|\mu_t(i)+\varepsilon)\}\ . 
        \label{eq:th1_13}
    \end{align}
    Finally, combining~\eqref{eq:th1_10} and~\eqref{eq:th1_13}, we recover~\eqref{eq:lemma_KL_UB}. The only thing left to prove is that $M_t(i)$ is a martingale satisfying $\E[M_1(i)] = 1$. Note that if $A_t\neq i$, we have $\E[M_t(i)\med\mcF_{t-1}] = M_{t-1}(i)$. If $A_t=i$, we have
    \begin{align}
        \E[M_t(i)\med\mcF_{t-1}] \;&=\; \E\left[ \displaystyle\bigintsss_{\Omega^{\otimes T_t(i)}} \E_{\eta} \left[ \displaystyle\prod\limits_{s\in[t]:A_s = i} \frac{\pi_i(X_s\med\rho)}{\pi_i(X_s\med\mu(i))}\right]\displaystyle\prod\limits_{s\in[t]:A_s=i}\pi_i(X_s\med\mu_t(i))\diff\mcX_t^i\;\Bigg\lvert\;\mcF_{t-1}\right]\\
        &=\; M_{t-1}(i)\cdot\E\left[ \displaystyle\bigintsss_{\Omega} \E_{\eta}\left[ \frac{\pi_i(X\med\rho)}{\pi_i(X\med\mu(i))}\right]\pi_i(X\med\mu_t(i))\diff X\right]\\
         &=\; M_{t-1}(i)\cdot\E_{\eta}\left[ \displaystyle\bigintsss_{\Omega} \E\left[ \frac{\pi_i(X\med\rho)}{\pi_i(X\med\mu(i))}\right]\pi_i(X\med\mu_t(i))\diff X\right]\\
         &=\; M_{t-1}(i)\cdot\E_{\eta}\left[ \displaystyle\bigintsss_{\Omega} \underbrace{\left(\displaystyle\bigintsss_{\Omega} \frac{\pi_i(X\med\rho)}{\pi_i(X\med\mu(i))} \pi_i(X\med\mu(i)) \diff X\right)}_{=1}  \pi_i(X\med\mu_t(i))\diff X\right]\\
         &=\;M_{t-1}(i)\cdot \E_{\eta}\left[ \displaystyle\bigintssss \pi_i(X\med\mu_t(i))\diff X\right]\\
         &=\;M_{t-1}(i)\ ,
    \end{align}
    which proves that $M_t(i)$ is a martingale. Furthermore, by definition, if $A_1\neq i$, $\E[M_1(i)] = 1$. Alternatively, if $A_1 = i$, we have
    \begin{align}
        \E[M_1(i)]\;&=\; \E_{\eta}\left[ \displaystyle\bigintsss_{\Omega} \underbrace{\left(\displaystyle\bigintsss_{\Omega} \frac{\pi_i(X\med\rho)}{\pi_i(X\med\mu(i))} \pi_i(X\med\mu(i)) \diff X\right)}_{=1}  \pi_i(X\med\mu_1(i))\diff X\right]\\
        &=\;\E_{\eta}[1]\\
        &=\;1\ .
    \end{align}
    This concludes the proof of Lemma~\ref{lemma:KL_UB}.
\end{proof}
\noindent Next, we delineate the choice of the threshold $\beta_t(\delta)$ that facilitates the $\delta-$PAC guarantee of the proposed stopping rule in~\eqref{eq:stop}. We have
\begin{align}
    \P_{\bnu} \Big ( \tau < +\infty,\; \hat A_{\tau} \neq a^\star\Big )\;&= \; \P_{\bnu}\Big (\exists t\in\N, i\neq a^\star : \bar\mu_t(i)\geq\max\limits_{j\neq i} \bar\mu_t(j),\;\min\limits_{j\neq i} \Lambda_t(i,j)\geq \beta_t(\delta)\Big )\\
    &\leq\;\sum\limits_{i\neq a^\star} \P_{\bnu}\left(\exists t\in\N : \bar\mu_t(i)\geq\max\limits_{j\neq i} \bar\mu_t(j),\;\min\limits_{j\neq i} \Lambda_t(i,j) \geq \beta_t(\delta) \right )\\
    & \leq\;\sum\limits_{i\neq a^\star} \P_{\bnu}\left( \exists t\in\N : \bar\mu_t(i)\geq\max\limits_{j\neq i}\bar\mu_t(j),\; \Lambda_t(i,a^\star) \geq \beta_t(\delta)\right)\\
    &\leq \;\sum\limits_{i\neq a^\star} \P_{\bnu}\Big( \exists t\in\N : T_t(i)d_i(\mu_t(i)\|\tilde\mu_t(a^\star)) + T_t(a^\star)d_{a^\star}(\mu_t(a^\star)\|\tilde\mu_t(a^\star))\geq \beta_t(\delta)\Big)\\
    &\leq\;\sum\limits_{i\neq a^\star} \P_{\bnu}\Big( \exists t\in\N : T_t(i)d_i(\mu_t(i)\|\mu(i)) + T_t(a^\star)d_{a^\star}(\mu_t(a^\star)\|\mu(a^\star))\geq \beta_t(\delta)\Big)\ ,
    \label{eq:th1_14}
\end{align}
where~\eqref{eq:th1_14} uses the definition of $\Lambda_t(i,a^\star)$. Specifically, using the KKT conditions, we obtain that the minimizer in~\eqref{eq:GLLR} satisfies the condition that $\rho(i)=\rho(a^\star)$. Denoting the minimizer by $\tilde\mu_t(a^\star)$, i.e., 
\begin{align}
    \tilde\mu_t(a^\star)\;\triangleq\;\argmin\limits_{x\in [\mu_(i),\mu_t(a^\star)]}\; \left\{T_t(i)d_i(\mu_t(i)\|x) + T_t(a^\star)d_{a^\star}(\mu_t(a^\star)\|x)\right\}\ ,
    \label{eq:th1_15}
\end{align}
we obtain~\eqref{eq:th1_14}. Furthermore $\mu(i)$ and $\mu(a^\star)$ satisfies the constraint in~\eqref{eq:GLLR}, which yields~\eqref{eq:th1_14}. Next, using Lemma~\ref{lemma:KL_UB}, \eqref{eq:th1_14} can be upper bounded as
\begin{align}
    \P_{\bnu}\Big(\tau<+\infty,\;\hat A_{\tau} \neq a^\star \Big)\;&\leq\; \sum\limits_{i\neq a^\star} \P_{\bnu}\bigg ( \exists t\in\N : \log \underbrace{M_t(i) M_t(a^\star)}_{\triangleq M_t} + \frac{1}{2}\log T_t(i) T_t(a^\star) + \frac{1}{2}\log\mcI_i(\mu_t(i))\mcI_{a^\star}(\mu_t(a^\star))\nonumber\\
    &\qquad+ 2\log\frac{|\Theta|}{\sqrt{2\pi}} + T_t(i)\max\{d_i(\mu_t(i)\|\mu_t(i)-\varepsilon), d_i(\mu_t(i)\|\mu_t(i)+\varepsilon) \} \nonumber\\
    &\qquad\qquad + T_t(a^\star)\max\{d_{a^\star}(\mu_t(a^\star)\|\mu_t(a^\star)-\varepsilon),d_{a^\star}(\mu_t(a^\star)\|\mu_t(a^\star)+\varepsilon)\}\nonumber\\
    &\qquad\qquad\qquad-W_t(\varepsilon,i) - W_t(a^\star)\geq \beta_t(\delta)\bigg )\\
    &\leq\; \sum\limits_{i\neq a^\star} \P_{\bnu}\bigg( \exists t\in\N : \log M_t + \log \frac{t}{2} + \max\limits_{i\in[K]}\mcI_i(\mu_t(i)) - 2\cdot\min\limits_{i\in[K]}W_t(\varepsilon,i)\nonumber\\
    &\quad + t\cdot\max\limits_{i\in[K]}\left\{\max\{d_i(\mu_t(i)\|\mu_t(i)-\varepsilon),d_i(\mu_t(i)\|\mu_t(i)+\varepsilon)\}\right\}+ 2\log\frac{|\Theta|}{\sqrt{2\pi}}\geq\beta_t(\delta)\bigg)\ ,
    \label{eq:th1_16}
\end{align}
where~\eqref{eq:th1_16} is obtained by using the AM-GM inequality. Next, recalling the definitions of $W_t(\varepsilon)$ in~\eqref{eq:W_t} and $\varepsilon_t$ in~\eqref{eq:varepsilon_t}, let us set
\begin{align}
    \beta_t(\delta)\;&\triangleq\; W_t(\varepsilon) + t\varepsilon_t + 2\log\frac{|\Theta|}{\sqrt{2\pi}} + \log\frac{t(K-1)}{2\delta}\ .
    \label{eq:th1_17}
\end{align}
Furthermore, note that $M_t$ is a martingale. To verify this, WLOG, let us assume that $A_t = i$. We have
\begin{align}
    \E[M_t\med\mcF_{t-1}]\;&=\; \E[M_t(i)\cdot M_t(a^\star)\med\mcF_{t-1}]\\
    & =\; M_{t-1}(a^\star)\cdot \E[M_t(i)\med\mcF_{t-1}]\\
    & = \; M_{t-1}(a^\star)\cdot M_{t-1}(i)\ .
    \label{eq:th1_18}
\end{align}
Furthermore, $\E[M_1] = E[M_1(a^\star)\cdot M_1(i)] = \E[M_1(a^\star)]\cdot\E[M_1(i)] = 1$. Finally, combining~\eqref{eq:th1_16}, \eqref{eq:th1_17} and~\eqref{eq:th1_18}, and using Ville's inequality, we obtain
\begin{align}
    \P_{\bnu}\Big(\tau<+\infty,\;\hat A_{\tau} \neq a^\star\Big) \;&\leq\; \sum\limits_{i\neq a^\star} \P_{\bnu} \Big(\exists t\in\N : \log M_t\geq\log\frac{K-1}{\delta} \Big)\\
    &\leq\;\sum\limits_{i\neq a^\star} \frac{\delta}{K-1}\E[M_1]\\
    &=\; \delta\ .
\end{align}
This concludes the proof.

\section{Problem Complexity}

\subsection{Proof of Lemma~\ref{lemma:simplified problem complexity}                        }
\label{proof: simplified problem complexity}
Recall that corresponding to a bandit instance $\bnu\in\mcM$ with the best arm $a^\star$, the set of alternate bandit instances is defined as
\begin{align}
    {\sf alt}(a^\star) \triangleq \left\{ \bar\bnu\in\mcM : m(\bar\P_{a^\star})\leq\max\limits_{i\neq a^\star}\; m(\bar\P_i)\right\}\ .
\end{align}
It can be readily verified that the set of alternate bandit instances can be equivalently stated as
\begin{align}
    {\sf alt}(a^\star) = \bigcup\limits_{i\neq a^\star} \Big\{\bar\bnu\in\mcM : m(\bar\P_i) \geq m(\bar\P_{a^\star})\Big\}\ .
    \label{eq:alternate}
\end{align}
Using (\ref{eq:alternate}), the problem complexity can be simplified as follows:
\begin{align}
    \Gamma(\bnu) &= \sup\limits_{\bw\in\Delta^K}\inf\limits_{\bar\bnu\in{\sf alt}(a^\star)}\; \sum\limits_{i\in[K]} w_i D_{\sf KL}(\P_i\|\bar\P_i)\\
    & = \sup\limits_{\bw\in\Delta^K}\;\min\limits_{i\neq a^\star}\;\inf\limits_{\bar\bnu\in\mcM : m(\bar\P_i)\geq m(\bar\P_{a^\star})}\;\Big \{ w_{a^\star}D_{\sf KL}(\P_{a^\star}\|\bar\P_{a^\star}) + w_iD_{\sf KL}(\P_i\|\bar\P_i)\Big \}\ .
\end{align}
Let us define
\begin{align}
    \Gamma_i(\bnu,\bw)\triangleq \inf\limits_{\bar\bnu\in\mcM : m(\bar\P_i)\geq m(\bar\P_{a^\star})}\;\Big \{ w_{a^\star}D_{\sf KL}(\P_{a^\star}\|\bar\P_{a^\star}) + w_iD_{\sf KL}(\P_i\|\bar\P_i)\Big \}\ .
\end{align}
It can be readily verified that
\begin{align}
    \Gamma(\bnu,\bw)\triangleq \min\limits_{i\neq a^\star}\;\Gamma_i(\bnu,\bw)\ .
\end{align}
Note that $\Gamma_i(\bnu,\bw)$ can be further simplified as:
\begin{align}
    \Gamma_i(\bnu,\bw) &= \inf\limits_{x\in\R}\left \{w_{a^\star}\inf\limits_{\bar\P\in\mcP(\Omega): m(\bar\P)\leq x}\; D_{\sf KL}(\P_{a^\star}\|\bar\P) + w_i \inf\limits_{\bar\P\in\mcP(\Omega): m(\bar\P)\geq x}\; D_{\sf KL}(\P_{i}\|\bar\P)\right\} \\
    &= \inf\limits_{x\in\R}\Big\{w_{a^\star}d_{\sf U}(\P_{a^\star},x) + w_id_{\sf L}(\P_i,x)\Big\}\ .
    \label{eq:alternate_pc}
\end{align}
Note that $d_{\sf U}$ is non-increasing in $x$ and $d_{\sf L}$ is non-decreasing in $x$. Thus, when $x<\mu(i)$, we have $d_{\sf U}(\P_{a^\star},x) > d_{\sf U}(\P_{a^\star},\mu(i))$, and $d_{\sf L}(\P_i,x)> d_{\sf L}(\P_i,\mu(i))$. This implies that the optimizer $x_i^\star$ of (\ref{eq:alternate_pc}) should satisfy $x_i^\star\geq \mu(i)$. Using a similar argument, we can show that $x_i^\star\leq \mu(a^\star)$. Hence, (\ref{eq:alternate_pc}) can be rewritten as:
\begin{align}
    \Gamma_i(\bnu,\bw) = \inf\limits_{x\in\left[\mu(i),\mu(a^\star)\right]}\Big\{w_{a^\star}d_{\sf U}(\P_{a^\star},x) + w_id_{\sf L}(\P_i,x)\Big\}\ .
    \label{eq:Gamma_i}
\end{align}
Finally, the problem complexity can be equivalently expressed as:
\begin{align}
    \Gamma(\bnu) = \sup\limits_{\bw\in\Delta^K}\;\min\limits_{i\neq a^\star}\;\inf\limits_{x\in\left[\mu(i),\mu(a^\star)\right]}\; \Big\{w_{a^\star}d_{\sf U}(\P_{a^\star},x) + w_id_{\sf L}(\P_i,x)\Big\}\ .
\end{align}

\subsection{Proof of Lemma~\ref{lemma:properties of problem complexity}}
\label{proof: properties of problem complexity}
\begin{enumerate}
    \item First, we will prove that $d_{\sf U}$ and $d_{\sf L}$ are strictly convex in $x$. For any $x\in\R$ and $y\in\R$, and for any $\lambda\in[0,1]$, let us define
    \begin{align}
        z \triangleq  \lambda x + (1-\lambda)y\ .
        \label{eq:z}
    \end{align}
    Furthermore, define
    \begin{align}
        \eta_x &\triangleq \arginf\limits_{\eta\in\mcP(\Omega): m(\eta)\leq x}\; D_{\sf KL}(\P_{a^\star}\| \eta)\ ,\nonumber\\
        \eta_y &\triangleq \arginf\limits_{\eta\in\mcP(\Omega): m(\eta)\leq y}\; D_{\sf KL}(\P_{a^\star}\| \eta)\ ,\nonumber\\ \text{and}\qquad \eta_z &\triangleq \arginf\limits_{\eta\in\mcP(\Omega): m(\eta)\leq z}\; D_{\sf KL}(\P_{a^\star}\| \eta)\ .
    \end{align}
    Furthermore, define $\kappa_z \triangleq \lambda\eta_x + (1-\lambda)\eta_y$. Note that
    \begin{align}
        m(\kappa_z) &= \lambda m(\eta_x) + (1-\lambda)m(\eta_y)\\
        &\leq \lambda x + (1-\lambda)y\\
        &\stackrel{(\ref{eq:z})}{=}z\ .
        \label{eq:convex1}
    \end{align}
    Now,
    \begin{align}
        d_{\sf U}(\P_{a^\star},z) & = D_{\sf KL}(\P_{a^\star}\| \eta_z)\\
        \label{eq:convex2}
        &\leq D_{\sf KL}(\P_{a^\star}\| \kappa_z) \\
        \label{eq:convex3}
        &< \lambda D_{\sf KL} (\P_{a^\star}\| \eta_x) + (1-\lambda) D_{\sf KL}(\P_{a^\star}\| \eta_y)\\
        & = \lambda d_{\sf U}(\P_{a\star}, x) + (1-\lambda) d_{\sf L}(\P_{a^\star},y)\ ,
    \end{align}
    where (\ref{eq:convex2}) is a result of (\ref{eq:convex1}), and (\ref{eq:convex3}) is a result of the strict convexity of KL divergence in both arguments. Thus, $d_{\sf U}$ is strictly convex in $x$. Using a similar argument, we can prove that $d_{\sf L}$ is also strictly convex in $x$. Thus, $g_i: \mcM\times \R \mapsto \R$, defined as
    \begin{align}
        g_i(\bnu,x)\triangleq w_{a^\star}d_{\sf U}(\P_{a^\star},x) + w_id_{\sf L}(\P_i,x)\ ,
    \end{align}
    is strictly convex in its second argument. Thus, $g_i$ has a unique minimum in $[\mu(i),\mu(a^\star)]$. 
    
    \item For establishing the continuity of $\Gamma : \mcM\mapsto \R$ and $\bw : \mcM\mapsto\Delta^K$, we will leverage the following two lemmas, which provide the sufficient conditions for continuity.
    \begin{lemma}[Berge's maximum theorem~\cite{sundaram1996first}] 
    \label{lemma:Berge}
        Suppose $g$ is a continuous function on $\mcS\times\Theta$ and $\mcD: \Theta\mapsto\mcS$ is a compact-valued continuous correspondence on $\Theta$. Let
        \begin{align}
            g^\star(\theta)\triangleq \max\limits_{x\in\mcD(\theta)}\; g(x,\theta)\quad
            \text{and}\qquad  \mcD^\star(\theta)\triangleq \argmax\limits_{x\in\mcD(\theta)}\; g(x,\theta)\ .
        \end{align}
        Then, $g^\star$ is a continuous function on $\Theta$, and $\mcD^\star$ is a compact-valued upper semicontinuous correspondence on $\Theta$.
    \end{lemma}
    \begin{lemma}[\cite{Lin91divergencemeasures}]
    \label{lemma:JS}
        Let us denote the generalized Jensen-Shannon (JS) divergence between two measures $\P_1$ and $\P_2$ with weight $\alpha\in(0,1)$ by
        \begin{align}
            {\sf JS}_\alpha(\P_1\|\P_2)\triangleq \alpha D_{\sf KL}(\P_1\|\bar\P_\alpha) + (1-\alpha) D_{\sf KL}(\P_2\|\bar\P_\alpha)\ ,
        \end{align}
        where we have defined 
        \begin{align}
            \bar\P_\alpha \triangleq \alpha \P_1 + (1-\alpha)\P_2\ .
        \end{align}
        Then, ${\sf JS}_\alpha$ is upper-bounded as
        \begin{align}
            {\sf JS}_\alpha(\P_1\|\P_2)\leq 1\ .
        \end{align}
    \end{lemma}
    Now, we show that $\Gamma(\bnu)$ and $\bw(\bnu)$ is continuous in $\bnu$. For this, let us define the correspondence $\mcD: \mcM\mapsto \Delta^K$ such that for any $\bnu\in\mcM$, $\mcD(\bnu)\triangleq \Delta^K$. For any $\bnu\in\mcM$, $\mcD(\bnu)$ is a compact set; hence, $\mcD$ is a compact-valued constant correspondence. Finally, we need to show that for each $i\in[K]\setminus\{a^\star\}$, $\Gamma_i(\bnu,\bw)$ is continuous in $\bnu$ and $\bw$, where we have defined $\Gamma_i(\bnu,\bw)$ in~(\ref{eq:Gamma_i}). First, note that $\Gamma_i$ is lower semicontinuous in $\bnu$ due to the lower semicontinuity of KL divergence in both arguments~\cite{posner1975random}. Next, we will leverage~\cite[Theorem 5.43]{hitchhiker}, which provides a sufficient condition for the global continuity of convex functions.
    \begin{lemma}[\cite{hitchhiker}]
    \label{lemma:usc}
        For a convex function $f : \mcX\mapsto \R$ on an open convex subset of a topological vector space, the following statements are equivalent.
        \begin{enumerate}
            \item $f$ is bounded above on a neighborhood of some point in $\mcX$.
            \item $f$ is upper semicontinuous on $\mcX$.
        \end{enumerate}
    \end{lemma}
    Note that $\Gamma_i$ is convex in its first argument since KL divergence is a convex function. Let us denote the interior of the set of distributions $\mcM$ by ${\rm int}(\mcM)$. For any $\boldsymbol\eta\in{\rm int}(\mcM)$, there exists a neighborhood $\mcN_r(\boldsymbol\eta)\subset{\rm int}(\mcM)$, where we have defined
    \begin{align}
        \mcN_r(\boldsymbol\eta)\triangleq \Big\{\blambda\in{\rm int}(\mcM) : D_{\sf TV}(\boldsymbol\eta\|\blambda)<r\Big\}\ .
    \end{align}
    Furthermore, for any $\bnu\in\mcN_r(\boldsymbol\eta)$ and $\bw\in\Delta^K$, let us define the distribution
    \begin{align}
        \kappa_{a^\star,i}\triangleq \frac{w_{a^\star}\P_{a^\star} + w_i\P_i}{w_{a^\star} + w_i}\ .
    \end{align}
    Expanding $\Gamma_i$, we obtain
    \begin{align}
        \Gamma_i(\bnu,\bw) & = \inf\limits_{\bar\bnu\in\mcM : m(\bar\P_{a^\star})\leq m(\bar\P_i)}\; \Big\{w_{a^\star}D_{\sf KL}(\P_{a^\star}\|\bar\P_{a^\star}) + w_iD_{\sf KL}(\P_i\| \bar\P_i)\Big\}\\
        &\leq w_{a^\star} D_{\sf KL}(\P_{a^\star}\|\kappa_{a^\star,i}) + w_i D_{\sf KL}(\P_i\|\kappa_{a^\star,i})\\
        &= (w_{a^\star} + w_i)\left( \frac{w_{a^\star}}{w_{a^\star}+w_i} D_{\sf KL}(\P_{a^\star}\|\kappa_{a^\star,i}) + \frac{w_i}{w_{a^\star}+w_i} D_{\sf KL}(\P_{i}\|\kappa_{a^\star,i})\right ) \\
        & = (w_{a^\star}+ w_i){\sf JS}_{\frac{w_{a^\star}}{w_{a^\star}+w_i}} (\P_{a^\star}\|\P_i)\\
        &\leq 1\ ,
        \label{eq:boundedJS1}
    \end{align}
    where~(\ref{eq:boundedJS1}) is a result of Lemma~\ref{lemma:JS}. Thus, leveraging Lemma~\ref{lemma:usc}, we obtain that $\Gamma_i(\bnu,\bw)$ is upper semicontinuous in $\bnu$, which proves that $\Gamma_i(\bnu,\bw)$ is continuous in its first argument in $\mcM$. Furthermore, $\Gamma_i(\bnu,\bw)$ is linear in the second argument, and hence, it is continuous. This shows that $\Gamma_i$ and the correspondence $\mcD$ satisfies the conditions in Lemma~\ref{lemma:Berge}, and we obtain that $\Gamma_i(\bnu,\bw)$ is continuous in $\bnu$ and $\bw$ is upper hemicontinuous in $\bnu$. Finally, following the same line of arguments as~\cite[Proposition 7]{russo2016}, we can show that for a given $\bnu\in\mcM$, $\bw$ is the unique allocation satisfying
    \begin{align}
        \Gamma_i(\bnu,\bw) = \Gamma_j(\bnu,\bw)\ , \quad {\rm for\; all}\quad i,j\neq a^\star\ .
    \end{align}
    This shows that $\bw$ is continuous in $\bnu$.
\end{enumerate}

\section{Proof of Theorems~\ref{theorem: convergence in allocation} and~\ref{theorem: convergence in allocation 2}}
\label{proof: convergence in allocation}
First, we show that the explicit exploration phase ensures that each arm $i\in[K]$ is sampled sufficiently often, such that the sample mean values converge to the true means. Let us define the time instant $N_{\bnu}^\epsilon$ as
\begin{align}
    N_{\bnu}^\epsilon\triangleq \inf\Big\{t\in\N : |\bar\mu_s(i)-\mu(i)|<\epsilon, \forall i\in[K],\;\forall s\geq t\Big\}\ .
\end{align}
The stochastic time $N_{\bnu}^\epsilon$ marks the convergence of the sample means to the respective ground truths for every arm $i\in[K]$. In the following result, we will show that the TCB and ITCB arm selection strategies ensure that $N_{\bnu}^\epsilon$ has a finite average value. This result is instrumental in showing the convergence in allocation for the TCB and ITCB sampling strategies stated in Theorem~\ref{theorem: convergence in allocation} and Theorem~\ref{theorem: convergence in allocation 2}.
\begin{theorem}[Convergence in mean]
\label{theorem: convergence in mean}
    Under the TCB and ITCB sampling strategies, we have $\E_{\bnu}[N_{\bnu}^\epsilon]<+\infty$.
\end{theorem}
\begin{proof}
    We use the notion of $r$-quick convergence, which we define below.
    \begin{definition}[$r$-quick convergence~\cite{Chow1978ProbabilityTI}]
        Consider the sequence of i.i.d. zero-mean random variables $\{Z_t : t\in\N\}$. Let $\bar Z_t \triangleq \frac{1}{t}\sum_{s=1}^t Z_t$ denote the empirical mean. Furthermore, for any $\epsilon\in\R_+$ define
        \begin{align}
            T_{\epsilon}\;\triangleq\; \sup\left\{ t\in\N : |\bar Z_t| > \epsilon\right\}\ .
        \end{align}
        Then, $\{Z_t : t\in\N\}$ converges $r$-quickly for $r>0$, if $\E[T_\epsilon^r]<+\infty$.
    \end{definition}
    \noindent We leverage $r$-quick convergence for $r=2$ to establish the convergence of the sample means of each arm $i\in[K]$ to the corresponding ground truth values. For this, we first state the necessary and sufficient condition for $r$-quick convergence to hold. 

    \begin{lemma}[Corollary $4$,~\cite{Chow1978ProbabilityTI}]
        The i.i.d. sequence $\{Z_t : t\in\N\}$ converges $r$-quickly if and only if $\E[|Z_t|^{r+1}]<+\infty$.
        \label{lemma:quick_conv}
    \end{lemma}
    \noindent For any arm $i\in[K]$, let us denote the realizations of $T_t(i)$ by $\ell_t(i)$. Furthermore, for any $s\in[t]$ such that $A_s=i$, let us set
    \begin{align}
        Z_s(i)\;\triangleq\;X_s(i) - \mu(i)\ , \qquad\text{and}\qquad \bar Z_t(i)\;\triangleq\; \frac{1}{T_t(i)}\sum\limits_{s\in[t]: A_s = i} \Big( X_s(i) - \mu(i)\Big)\ .
    \end{align}
    Let us define
    \begin{align}
        T_{\epsilon}(i)\;\triangleq\;\sup\left\{ \ell_t(i)\in\N : |\bar Z_t(i)| > \epsilon\right\}\ .
    \end{align}
    Using Assumption 4 along with Lemma~\ref{lemma:quick_conv}, we have
    \begin{align}
        \E_{\bnu}\left [ \big( T_\epsilon(i)\Big)^2\right]\;<\;+\infty\ .
        \label{eq:qc1}
    \end{align}
    Furthermore, note that owing to the explicit exploration of the TCB and ITCB algorithms, leveraging~\cite[Lemma 8]{pmlr-v49-garivier16a}, for any arm $i\in[K]$, we have 
    \begin{align}
        T_t(i)\;\geq\;\sqrt{\frac{t}{K}} - 1\ .
    \end{align}
    Accordingly, when arm $i\in[K]$ has been sampled $T_\epsilon(i)$ times, the following inequality holds
    \begin{align}
        T_\epsilon(i)\;\geq\;\sqrt{\frac{N_0}{K}} - 1\ ,
        \label{eq:qc2}
    \end{align}
    where $N_0$ denotes the time instant at which the arm $i\in[K]$ has been sampled $T_\epsilon(i)$ times. Furthermore, assuming that $N_0\geq 4K$, we have
    \begin{align}
        T_{\epsilon}(i)\;\geq\;\sqrt{\frac{N_0}{2K}}\ ,
    \end{align}
    which yields that $N_0\leq 2K(T_\epsilon(i))^2$. Furthermore, from~\eqref{eq:qc1} we have $\E_{\bnu}[(T_{\epsilon}(i))^2]<+\infty$. The proof is completed by setting $N_{\bnu}^\epsilon = N_0$.
\end{proof}
\noindent Next, let us define the set of \emph{over-sampled} arms as:
\begin{align}
    \mcO_t^\epsilon\;\triangleq \; \left\{i\in[K] : \frac{T_t(i)}{t} > w_i(\bnu) + \epsilon\right\}\ .
\end{align}
Furthermore, we define the set of \emph{under-sampled} arms as:
\begin{align}
    \mcP_t^\epsilon \;\triangleq\; \left\{i\in[K] : \frac{T_t(i)}{t} < w_i(\bnu) + \frac{\epsilon}{2}\right\}\ . 
\end{align}
The convergence in allocation for the proposed algorithm is shown in two key steps. First, we prove that if any sampling strategy always samples from the set of under-sampled arms, then the sampling strategy converges to the optimal allocation $\bw(\bnu)$. This step is common in the proof for both Theorem~\ref{theorem: convergence in allocation} and Theorem~\ref{theorem: convergence in allocation 2}. In the next step, we show that the proposed sampling strategies always sample from the set of under-sampled arms. We show the first step through Lemma~\ref{lemma:con_alloc_1} and Lemma~\ref{lemma:con_alloc_2}, which we provide next. Essentially, Lemma~\ref{lemma:con_alloc_1} shows that if the sampling strategy always samples from the set of under-sampled arms, then, after some time, the set of over-sampled arms becomes empty. Lemma~\ref{lemma:con_alloc_2} then shows that when the set of over-sampled arms is empty, eventually, the allocation for each arm converges to the optimal allocation. The key distinction in the proofs of Theorem~\ref{theorem: convergence in allocation} and Theorem~\ref{theorem: convergence in allocation 2} arises in the next step. In Lemma~\ref{lemma:con_alloc_3}, we show that the TCB arm selection rule stated in~(\ref{eq: sampling rule}) always samples from the set of under-sampled arms. In Lemma~\ref{lemma:con_alloc_4}, we show that the ITCB arm selection rule provided in~(\ref{eq: sampling rule 2}) always samples from the set of under-sampled arms. Before stating Lemma~\ref{lemma:con_alloc_1}, let us define the sampling proportion $\bgamma_t\triangleq [\gamma_{t,1},\cdots,\gamma_{t,K}]$ computed at the current MLE $\bmu_t$ as
\begin{align}
    \bgamma_t\;\triangleq\; \argsup\limits_{\bw\in\Delta^K}\inf\limits_{\bar\bnu\in{\sf alt}(a^\star)}\; \sum\limits_{i\in[K]} w_iD_{\sf KL}(\P_{t,i}\|\bar\P_i)\ .
\end{align}

\begin{lemma}
\label{lemma:con_alloc_1}
There exists a stochastic time $N^\epsilon\in\N$ such that for all $t>N^{\epsilon}$, $\mcO_t^\epsilon = \emptyset$, and $\E[N^{\epsilon}]<+\infty$, if the sampling strategy satisfies $\frac{1}{t}T_t(a_{t+1}) < \gamma_{t,a_{t+1}} + \zeta$ for any $\zeta\in[0,\frac{\epsilon}{4}]$ and for any $t>M$, where $M$ is a stochastic time satisfying $\E_{\bnu}[M]<+\infty$.
\end{lemma}
\begin{proof}
    Let us define the time instant $M_1^\epsilon$ such that for all $t>M_1^\epsilon$ and for all $i\in[K]$, $|\gamma_{t,i} - w_i(\bnu)| < \epsilon/8$. Leveraging the continuity of $\Gamma$ in Lemma~\ref{lemma:properties of problem complexity} and the convergence of the MLE in Theorem~\ref{theorem: convergence in mean}, we obtain that $\E[M_1^\epsilon]<+\infty$. Furthermore, define $M_2^\epsilon \triangleq \lceil (8/\epsilon) - 1\rceil$, and $M^{\epsilon} \triangleq \max\{M,M_1^\epsilon, M_2^\epsilon\}$. We have the following two cases:
    \begin{enumerate}
        \item $\mcO_{M^{\epsilon}}^\epsilon = \emptyset$: In this case, we will use induction on $t$ to show that for all $t>M^{\epsilon}$, $\mcO_t^{\epsilon} = \emptyset$. First, by our assumption, for $t=M^\epsilon$, $\mcO_{M^{\epsilon}}^\epsilon = \emptyset$. Next, assume the inductive hypothesis that for some $t>M^{\epsilon}$, $\mcO_t^\epsilon=\emptyset$. Then,
        \begin{align}
            \frac{T_{t+1}(a_{t+1})}{t+1} &= \frac{T_t(a_{t+1})+1}{t+1}\\
            &<\frac{T_t(a_{t+1})}{t} + \frac{1}{t+1}\\
            \label{eq:con_alloc1}
            &< \gamma_{t,a_{t+1}} + \frac{1}{t+1} + \zeta\\
            &\leq w_{a_{t+1}}(\bnu) + \frac{\epsilon}{8} + \frac{1}{t+1} + \zeta\\
            &\leq w_{a_{t+1}}(\bnu) + \frac{\epsilon}{2}\ ,
            \label{eq:con_alloc2}
        \end{align}
        where~(\ref{eq:con_alloc1}) holds since the sampling strategy satisfies $\frac{1}{t}T_t(a_{t+1}) < \gamma_{t,a_{t+1}} + \zeta$, and~(\ref{eq:con_alloc2}) is obtained using the definition of $M_{\epsilon}$. Hence, $\mcO_{t+1}^\epsilon = \emptyset$, and it concludes the proof.
        \item $|\mcO_{M^{\epsilon}}^\epsilon|\geq 1$: In this case, following the same steps as~(\ref{eq:con_alloc1})-(\ref{eq:con_alloc2}), we can show that for all $t>M^{\epsilon}$, any $i\in\mcP_t^\epsilon$ is not included in $\mcO_t^\epsilon$. Furthermore, for any $t>M^{\epsilon}$ and for any $j\in\mcO_t^\epsilon$, let us define $L_j^\epsilon$ as the time that $j$ leaves $\mcO_t^\epsilon$, i.e., for all $t\in\{M^{\epsilon},\cdots, L_j^\epsilon-1\}$, $j\in\mcO_t^\epsilon$. Next, defining $L^\epsilon\triangleq \max_{j\in[K]} L_j^\epsilon$, for all $t>L^\epsilon$, we obtain that $|\mcO_t^\epsilon| = 0$. Finally, defining $N^\epsilon\triangleq \max\{M^{\epsilon},L^{\epsilon}\}$, we obtain that for all $t>N^{\epsilon}$, $\mcO_{t}^\epsilon = \emptyset$.
    \end{enumerate}
\end{proof}

\begin{lemma}
\label{lemma:con_alloc_2}
For all $t>N^{\frac{\epsilon}{K}}$, the allocation for every arm $i\in[K]$ satisfies
\begin{align}
\label{eq:con_alloc3}
    \left\lvert \frac{T_t(i)}{t} - w_i(\bnu)\right\rvert\;\leq\;\epsilon\ .
\end{align}
\end{lemma}
\begin{proof}
    We will prove~(\ref{eq:con_alloc3}) by contradiction. Assume that there exists $j\in[K]$ such that $\frac{1}{t}T_t(i) < w_j(\bnu) - \epsilon$. For all $t>N^{\frac{\epsilon}{K}}$, leveraging Lemma~\ref{lemma:con_alloc_1}, we have
    \begin{align}
        \sum\limits_{i\in[K]} \frac{T_t(i)}{t} &= \sum\limits_{i\neq j} \frac{T_t(i)}{t} + \frac{T_t(j)}{t}\\
        &\leq \sum\limits_{i\neq j} \left( w_i(\bnu) + \frac{\epsilon}{K} \right) + w_j(\bnu) - \epsilon\\
        & = 1 - \frac{\epsilon}{K}\ ,
    \end{align}
    which is a contradiction. Thus,~(\ref{eq:con_alloc3}) holds for all $t>N^{\frac{\epsilon}{K}}$. 
\end{proof}

\noindent Next, we show that our proposed sampling strategy always samples from the set of under-sampled arms, which is specified in Lemma~\ref{lemma:con_alloc_3}. Let us define the minimum sub-optimality gap
\begin{align}
    \Delta_{\min}\;\triangleq\;\min\limits_{i\in[K]\setminus\{a^\star\}}\;\mu(a^\star)-\mu(i)\ .
\end{align}

\begin{lemma}
\label{lemma:con_alloc_3}
For all $t>N_{\bnu}^{\Delta_{\min}/4}$, the TCB sampling rule provided in~(\ref{eq: sampling rule}) satisfies
\begin{align}
    \frac{T_t(a_{t+1})}{t} \leq \gamma_{t,a_{t+1}}\ .
\end{align}
\end{lemma}
\begin{proof}
Note that by Theorem~\ref{theorem: convergence in mean}, for all $t>N_{\bnu}^{\Delta_{\min}/4}$, the proposed sampling rule satisfies that $a_t^{\sf top} = a^\star$. Accordingly, the TCB sampling rule samples between any two arms, either the best arm $a^\star$ or the arm $a_t^{\min}$. Furthermore, note that both the arms $a^\star$ and $a_t^{\min}$ cannot be simultaneously over-sampled. To see why this is true, assume without loss of generality that arm $a^\star$ is over-sampled, i.e., $\frac{1}{t}T_t(a^\star) > \gamma_{t,a^\star}$. $\Gamma_t(\bw)$ is a minimum of linear functions, and hence, it is a concave function~\cite{boyd2004convex}. Furthermore, $\bw$ belongs to a compact space $\Delta^K$, and $\Gamma_t(\bw)$ has a \emph{unique} maxima~\cite{russo2016}. We will prove that $\frac{1}{t}T_t(a_t^{\min})<\gamma_{t,a_t^{\min}}$ by contradiction. Let us assume that $\frac{1}{t}T_t(a_t^{\min})>\gamma_{t,a_t^{\min}}$. We have
\begin{align}
    \label{eq:sr1_1}
    \min\limits_{x\in I_{t,a_t^{\min}}}  & \left\{\frac{T_t(a^\star)}{t} d_{\sf U}(\P_{t,a^\star},x) + \frac{T_t(a_t^{\min})}{t}d_{\sf L}(\P_{t,a_t^{\min}},x)\right\}\nonumber\\
    &\quad\geq \min\limits_{x\in I_{t,a_t^{\min}}}\left\{\gamma_{t,a^\star} d_{\sf U}(\P_{t,a^\star},x) + \frac{T_t(a_t^{\min})}{t}d_{\sf L}(\P_{t,a_t^{\min}},x)\right\}\\
    \label{eq:sr1_11}
    &\quad\geq \min\limits_{x\in I_{t,a_t^{\min}}}\left\{\gamma_{t,a^\star} d_{\sf U}(\P_{t,a^\star},x) + \gamma_{t,a_t^{\min}}d_{\sf L}(\P_{t,a_t^{\min}},x)\right\}\\
    &\quad = \Gamma_t(\bgamma_t)\ ,
    \label{eq:sr1_2}
\end{align}
where~(\ref{eq:sr1_1}) and~(\ref{eq:sr1_11}) hold due to the fact that $\Gamma_t$ is an increasing function in each of its arguments, keeping the other argument fixed~\cite[Lemma 2]{russo2016}. It can be readily verified that~(\ref{eq:sr1_2}) is a contradiction, since we obtain that $\Gamma_t(\frac{1}{t}\bT_t)\geq \Gamma_t(\bgamma_t)$, where $\bgamma_t$ is the \emph{unique} maximizer. Thus, we have $\frac{1}{t}T_{t}(a_t^{\min}) < \gamma_{t,a_t^{\min}}$, if $\frac{1}{t}T_t(a^\star) > \gamma_{t,a^\star}$. Next, for any $t>N_{\bnu}^{\Delta_{\min}/4}$, based on the TCB arm selection strategy, we have the following two cases.
\begin{enumerate}
    \item $a_{t+1}=a^\star$: 
    We will prove that $\frac{1}{t}T_t(a^\star)\leq \bgamma_{t,a^\star}$ by contradiction. We proceed with our assumption that $\frac{1}{t}T_t(a^\star)>\gamma_{t,a^\star}$, and define the points $\bz$ and $\bz^\prime$ such that for any $\lambda_1,\lambda_2\in(0,1)$,
    \begin{align}
        \bz\;&\triangleq\; \lambda_1\bgamma_t + (1-\lambda_1)\bw_t^\prime\ ,\\
        \text{and}\;\; \frac{1}{t}\bT_t\;&= \; \lambda_2\bz^\prime + (1-\lambda_2)\bw_t^\prime\ .    
    \end{align}
    For a geometric representation of the relative position of these points, we refer to Figure~\ref{fig:supp_fig}.
    \begin{figure}[t]
    \centering
    \includegraphics[width=0.5\linewidth]{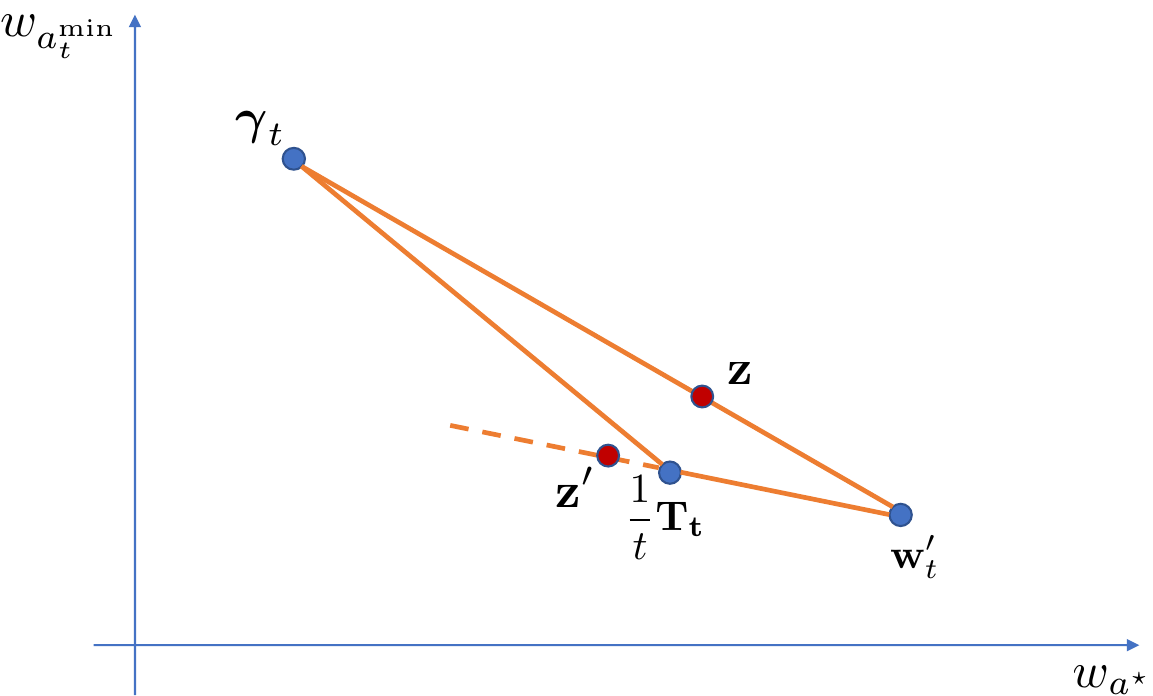}
    \caption{Positions of $\bz$ and $\bz^\prime$ assuming $\frac{1}{t}T_t(a^\star)>\gamma_{t,a^\star}$}
    \label{fig:supp_fig}
    \end{figure}
    Owing to the concavity of $\Gamma_t$, we have
    \begin{align}
        \Gamma_t(\bz)&\geq \lambda_1\Gamma_t(\bgamma_t) + (1-\lambda_1)\Gamma_t(\bw_t^\prime)\\
        &\geq \Gamma_t(\bw_t^\prime)\ .
        \label{eq:new_1}
    \end{align}
    Note that as a result of the TCB sampling rule, we have
    \begin{align}
    \label{eq:sr1_3}
        \Gamma_t(\bw_t^\prime)\;>\;\Gamma_t\left(\frac{1}{t}\bT_t\right)\ .
    \end{align}
    Accordingly, let us define 
    \begin{align}
        \epsilon_t\;\triangleq\; \Gamma_t(\bw_t^\prime) - \Gamma_t\left(\frac{1}{t}\bT_t\right)\ .
        \label{eq:epsilon_t}
    \end{align}
    Furthermore, for any $\bw\in\Delta^K$, let $\nabla\Gamma_t(\bw)$ denote the sub-gradient of the function $\Gamma_t$ at $\bw$. Owing to the concavity of $\Gamma_t$, we have
    \begin{align}
        \Gamma_t(\bz^\prime)\;&\geq\; \Gamma_t(\bz) - \left\langle \nabla\Gamma_t(\bz),\bz^\prime - \bz\right\rangle\\
        \label{eq:new2}
        &\geq\; \Gamma_t(\bz) - \left\lVert \nabla\Gamma_t(\bz)\right\rVert \left\lVert \bz^\prime - \bz\right\rVert\\
        &\stackrel{(\ref{eq:new_1})}{\geq}\; \Gamma_t(\bw_t^\prime) - \left\lVert \nabla\Gamma_t(\bz)\right\rVert \left\lVert \bz^\prime - \bz\right\rVert\\
        &\stackrel{(\ref{eq:epsilon_t})}{=}\; \Gamma_t\left(\frac{1}{t}\bT_t\right) + \epsilon_t - \left\lVert \nabla\Gamma_t(\bz)\right\rVert \left\lVert \bz^\prime - \bz\right\rVert\ ,
        \label{eq:new3}
    \end{align}
    where~(\ref{eq:new2}) is obtained using the Cauchy–Schwarz inequality. Furthermore, leveraging the concavity of $\Gamma_t$, we have
    \begin{align}
        \Gamma_t\left(\frac{1}{t}\bT_t\right)&\;\geq\; \lambda_2\Gamma_t(\bz^\prime) +  (1-\lambda_2)\Gamma_t(\bw_t^\prime)\\
        &\stackrel{(\ref{eq:new3})}{\geq}\; \lambda_2\left ( \Gamma_t\left(\frac{1}{t}\bT_t\right) + \epsilon_t -\left\lVert \nabla\Gamma_t(\bz)\right\rVert \left\lVert \bz^\prime - \bz\right\rVert \right ) + (1-\lambda_2)\Gamma_t(\bw_t^\prime)\ ,
    \end{align}
    which implies that
    \begin{align}
        \Gamma_t\left(\frac{1}{t}\bT_t\right)\;\geq\Gamma_t(\bw_t^\prime) - \frac{\lambda_2}{1-\lambda_2}\left ( \left\lVert \nabla\Gamma_t(\bz)\right\rVert \left\lVert \bz^\prime - \bz\right\rVert - \epsilon_t\right )\ .
        \label{eq:new4}
    \end{align}
    Let us set $\lambda_2 \triangleq O(\epsilon_t^2)$. Hence, (\ref{eq:new4}) can be rewritten as
    \begin{align}
        \Gamma_t\left(\frac{1}{t}\bT_t\right)&\;\geq\; \Gamma_t(\bw_t^\prime) + O(\epsilon_t^2)\\
        &\stackrel{(\ref{eq:epsilon_t})}{=}\; \Gamma_t\left(\frac{1}{t}\bT_t\right) + \epsilon_t + O(\epsilon_t^2)\ ,
    \end{align}
    which is a contradiction. 
    This shows that when $a_{t+1}=a^\star$, we have $\frac{1}{t}T_t(a_{t+1})\leq\gamma_{t,a_{t+1}}$.
    \item $a_{t+1}=a_t^{\min}$: Let us assume that $\frac{1}{t}T_t(a_t^{\min}) >\gamma_{t,a_t^{\min}}$. Furthermore, by our sampling strategy, we have
    \begin{align}
        \Gamma_t(\bw_t^\prime)\;<\;\Gamma_t\left(\frac{1}{t}\bT_t\right)\ .
    \end{align}
    Following similar arguments as the case when $a_{t+1}=a^\star$, leveraging the concavity of $\Gamma_t$, we can arrive at a contradiction. Thus, in this case, when $a_{t+1}=a_t^{\min}$, we have $\frac{1}{t}T_t(a_{t+1})\leq \gamma_{t,a_{t+1}}$.

\end{enumerate}
\end{proof}

\begin{lemma}
\label{lemma:con_alloc_4}
There exists a stochastic time $M_{\sf ITCB}$ such that $\E_{\bnu}[M_{\sf ITCB}]<+\infty$, and for all $t>M_{\sf ITCB}$, the ITCB sampling rule provided in~(\ref{eq: sampling rule 2}) with the sequence $r_t = \frac{\epsilon}{t}$ for any $\epsilon\in\R_+$ satisfies
\begin{align}
    \frac{T_t(a_{t+1})}{t} \leq \gamma_{t,a_{t+1}}\ .
\end{align}
\end{lemma}

\begin{proof}
Let us recall that
\begin{align}
    b_t^{\min}\;\triangleq\;\argmin\limits_{i\in[K]\setminus\{a_t^{\sf top}\}}\; \left\{ \min\limits_{x\in I_{t,i}}\left\{\frac{T_t(a_t^{\sf top})}{t}d_{\sf U}(\P_{t,a_t^{\sf top}},x)+\frac{T_t(i)}{t}d_{\sf L}(\P_{t,i},x)\right\} + \frac{\log(T_t(i))}{t}\right\}\ .
\end{align}
Note that for all $t>N_{\bnu}^{\Delta_{\min}/4}$, we have $a_t^{\sf top}=a^\star$. Furthermore, for all $t>\lceil \frac{1}{(K\epsilon^2)}\rceil$ and for all $i\in[K]$, we almost surely have
\begin{align}
    \label{eq:newi_1}
    \frac{\log T_t(i)}{t}\;&\leq\; \frac{\log (\sqrt{t/K}-1)}{t}\\
    &\leq\; \frac{\log (\sqrt{t/K})+1}{t}\\
    \label{eq:newi_2}
    &\leq\; \frac{\sqrt{t/K}}{t}\\
    \label{eq:newi_3}
    &\leq\;\epsilon\ ,
\end{align}
where~(\ref{eq:newi_1}) is a result of the fact that $T_t(i)\geq \sqrt{t/K}-1$ for all $i\in[K]$~\cite[Lemma 4]{mukherjee2022},~(\ref{eq:newi_2}) holds due to the fact that $\log(1+x)\leq x$, and~(\ref{eq:newi_3}) is obtained from the fact that $t>\lceil 1/(K\epsilon^2)\rceil$. Let us define $M_3^\epsilon\triangleq \max\{N_{\bnu}^{\Delta_{\min}/4}, 1/(K\epsilon^2)\rceil\}$. For a sufficiently small $\epsilon\in\R_+$, for any $t>M_3^\epsilon$, we almost surely have
\begin{align}
    b_t^{\min}\;&=\; \argmin\limits_{i\in[K]\setminus\{a^\star\}}\; \left\{ \min\limits_{x\in I_{t,i}}\left\{\frac{T_t(a^\star)}{t}d_{\sf U}(\P_{t,a^\star},x)+\frac{T_t(i)}{i}d_{\sf L}(\P_{t,i},x)\right\}+\epsilon\right\} \\
    & = \; a_t^{\min}\ .
\end{align}
% we almost surely have $\log(T_{t,i})/t\leq\epsilon$ due to the fact that $T_t(i)\geq \sqrt{t/K}-1$ for every $i\in[K]$. Let us define $M_3^\epsilon\triangleq\max\{N_{\bnu}^{\Delta_{\min}/4},\lceil 1/(K\epsilon^2)\rceil\}$. For any sufficiently small $\epsilon\in\R_+$ and for any $t>M_3^\epsilon$, we have
% \begin{align}
%     b_t^{\min}\;&=\; \argmin\limits_{i\in[K]\setminus\{a^\star\}}\; \left\{ \min\limits_{x\in I_{t,i}}\left\{\frac{T_t(a^\star)}{t}d_{\sf U}(\P_{t,a^\star},x)+\frac{T_t(i)}{i}d_{\sf L}(\P_{t,i},x)\right\}\right\} + \epsilon \\
%     & = \; a_t^{\min}\ .
% \end{align}
We have the following two cases.
\begin{enumerate}
    \item $a_{t+1}=a^\star:$ Let us assume that $\frac{1}{t}T_t(a^\star)>\gamma_{t,a^\star}+\zeta$. Due to the ITCB sampling strategy in~(\ref{eq: sampling rule 2}), for all $t>M_3^\epsilon$, we have
    \begin{align}
        \Gamma_t(\bw_t^\prime) + \frac{\log(tw_t^\prime(a_t^{\min}))}{t}\;&\geq\; \Gamma_t\left(\frac{1}{t}\bT_t\right) + \frac{\log(T_t(a_t^{\min}))}{t}\ ,
    \end{align}
    or, equivalently 
    \begin{align}    
        \Gamma_t(\bw_t^\prime) - \Gamma_t\left(\frac{1}{t}\bT_t\right)&\geq \frac{1}{t}\log\left (1+\frac{\frac{tr_t}{K-1}}{T_t(a_t^{\min})-\frac{tr_t}{K-1}}\right )\ .
        %\implies\;\; \Gamma_t(\bw_t^\prime) - \Gamma_t\left(\frac{1}{t}\bT_t\right)&\geq 0 \ .
    \end{align}
    This implies that
    \begin{align}
        \Gamma_t(\bw_t^\prime) - \Gamma_t\left(\frac{1}{t}\bT_t\right)&> 0 \ .
        \label{eq:sr2_1}
    \end{align}
    Following the same argument as Lemma~\ref{lemma:con_alloc_3}, (\ref{eq:sr2_1}) 
    implies that $\frac{1}{t}T_t(a^\star)\leq \gamma_{t,a^\star}+\zeta$ for any $\zeta\geq 0$.

    \item $a_{t+1}=a_t^{\min}$: Let us assume that $\frac{1}{t}T_t(a_t^{\min}) > \gamma_{t,a_t^{\min}}$. We will show that this is a contradiction if the condition in the ITCB sampling strategy in~(\ref{eq: sampling rule 2}) holds. Specifically, according to the ITCB sampling strategy, we have
    \begin{align}
        \Gamma_t\left(\frac{1}{t}\bT_t\right) - \Gamma_t(\bw_t^\prime)\;\geq\; -\frac{1}{t}\cdot\underbrace{\log\left (1+\frac{\frac{tr_t}{K-1}}{T_t(a_t^{\min})-\frac{tr_t}{K-1}}\right )}_{\triangleq\xi_t}\ .
        \label{eq:newi_4}
    \end{align}
    We may have the following two cases.
    \begin{itemize}
        \item $\Gamma_t\left(\frac{1}{t}\bT_t\right)- \Gamma_t(\bw_t^\prime)>0$: In this case, following the same line of arguments as in Lemma~\ref{lemma:con_alloc_3} we obtain that $\frac{1}{t}T_t(a_t^{\min})\leq \gamma_{t,a_t^{\min}}$. 
        \item $\Gamma_t\left(\frac{1}{t}\bT_t\right)\in[\Gamma_t(\bw_t^\prime) - \frac{1}{t}\xi_t, \Gamma_t(\bw_t^\prime)]$: Let us define the vector $\be_{i}\triangleq[e(1),\cdots,e(K)]^\top$, where, for any $j\in[K]$ we have defined
        \begin{align}
                &e(j)\triangleq \left\{
        	\begin{array}{ll}
        	-1, & \mbox{if} \;\; j=a^\star\\
        	1 , & \mbox{if}\;\;j\neq a^\star\\
        	\end{array}\right. \ .
        \end{align}
        Leveraging the concavity of $\Gamma_t$, we have
        \begin{align}
            \Gamma_t(\bw_t^\prime)\;\geq\; \Gamma_t\left(\frac{1}{t}\bT_t\right) - \underbrace{\left\langle \nabla\Gamma_t(\bw_t^\prime), \frac{1}{t}\bT_t - \bw_t^\prime\right\rangle}_{<0}\ ,
        \end{align}
        which implies that
        \begin{align}
            \Gamma_t(\bw_t^\prime) - \Gamma_t\left(\frac{1}{t}\bT_t\right)\;&\geq\; \left\lvert\left\langle   \nabla\Gamma_t(\bw_t^\prime), \frac{1}{t}\bT_t-\bw_t^\prime \right\rangle\right\rvert\\
            &=\; r_t\cdot \left\lvert\left\langle   \nabla\Gamma_t(\bw_t^\prime), \be\right\rangle\right\rvert\ .
            \label{eq:newi_5}
        \end{align}
        Combining~(\ref{eq:newi_4}) and~(\ref{eq:newi_5}), we obtain
        \begin{align}
            \xi_t\;&\geq\;tr_t\left\lvert\left\langle   \nabla\Gamma_t(\bw_t^\prime), \be\right\rangle\right\rvert \ .
            %&\geq\;\inf\limits_{\bw\in\Delta^K: |\gamma_{t,a_t^{\min}}-w_{a_t^{\min}}|\geq \zeta}\;\left\lvert\left\langle \nabla\Gamma_t(\bw),\be_{a_t^{\min}}\right\rangle\right\rvert\ .
            \label{eq:newi_10}
        \end{align}
        Next, let us define the set
        \begin{align}
            \mcM_+\;\triangleq\; \left\{ \bw\in\Delta^K : \left\lvert\left\langle \nabla\Gamma_t(\bw),\be\right\rangle\right\rvert>0\right\}\ .
        \end{align}
        It can be readily verified that $\bw_t^\prime\in\mcM_+$. Thus, from~(\ref{eq:newi_10}), we obtain that 
        \begin{align}
            \frac{1}{tr_t}\xi_t\;&\geq\; \inf\limits_{\bw\in\mcM_+}\; \left\lvert \left\langle \nabla\Gamma_t(\bw),\be\right\rangle\right\rvert\ ,
            \label{eq:newi_11}
        \end{align}
        % where we have used the convention that $\inf\{\emptyset\}=+\infty$. (\ref{eq:newi_11}) can be further lower-bounded as
        % \begin{align}
        %     \epsilon_t\;\geq\;\min\limits_{i\in[K]\setminus\{a^\star\}}\;\inf\limits_{\bw\in\mcM_{\zeta}(i)}\; \left\lvert \left\langle \nabla\Gamma_t(\bw),\be_i\right\rangle\right\rvert\ .
        %     \label{eq:newi_12}
        % \end{align}
        % Furthermore, define the instant $M_4^\epsilon\triangleq \max\{\lceil 8K/\epsilon^2, 3K/(3-2\sqrt{2})\rceil\}$. For all $t>M_4^\epsilon$, we have
        % \begin{align}
        %     \label{eq:newi_6}
        %     \epsilon_t\;&\leq\; \frac{1}{T_t(a_t^{\min})}\\
        %     \label{eq:newi_7}
        %     &\leq\; \frac{1}{\sqrt{t/K}-1}\\
        %     \label{eq:newi_8}
        %     &\leq\; \sqrt{\frac{2K}{t}}\\
        %     \label{eq:newi_9}
        %     &\leq\;\frac{\epsilon}{2}\ ,
        % \end{align}
        % where~(\ref{eq:newi_6}) is obtained from the fact that $\log(1+x)\leq x$,~(\ref{eq:newi_7}) follows from the property of forced exploration that $T_t(i)\geq \sqrt{t/K}-1$ for all $i\in[K]$~\cite[Theorem 4]{mukherjee2022}, and~(\ref{eq:newi_8}) and~(\ref{eq:newi_9}) follow from the fact that $t>M_4^\epsilon$. 
        Next, setting $r_t = \frac{\epsilon}{t}$, we obtain
        \begin{align}
            \frac{1}{tr_t}\xi_t\;&=\; \frac{1}{\epsilon}\log \left (1 - \frac{\epsilon}{K-1}\cdot \frac{1}{T_t(a_t^{\min}) - \frac{\epsilon}{K-1}} \right )\\
            &\leq \frac{1}{\epsilon}\underbrace{\log \left (1 - \frac{\epsilon}{K-1}\cdot \frac{1}{(\sqrt{t/K} - 1) - \frac{\epsilon}{K-1}} \right )}_{\triangleq g(t)}\ ,
            \label{eq:new_ITCB1}
        \end{align}
        where~\eqref{eq:new_ITCB1} follows from the property of forced exploration that $T_t(i)\geq \sqrt{t/K}-1$ for all $i\in[K]$~\cite[Lemma 4]{mukherjee2022}. Next, note that for all $t>K$, the function $g(t)$ is a monotonically decreasing function in $t$. Hence, there exists $M_4^\epsilon\in\N$ such that for all $t>M_4^\epsilon$, we have $g(t)\leq \epsilon^2$. Hence, for all $t>M_4^\epsilon$, we have
        \begin{align}
            \frac{1}{tr_t}\xi_t\;&\leq\;\epsilon\ .
        \end{align}
        Furthermore, setting 
        \begin{align}
            \epsilon\;\triangleq\;\inf\limits_{\bw\in\mcM_+}\; \left\lvert \left\langle \nabla\Gamma_t(\bw),\be\right\rangle\right\rvert\ ,
        \end{align}
        it can be readily verified that~(\ref{eq:newi_10}) is a contradiction. This implies that for all $t>M_4^{\epsilon}$, assuming that $\frac{1}{t}T_t(a_t^{\min})>\gamma_{t,a_t^{\min}}$, the ITCB sampling condition~(\ref{eq:newi_4}) does not hold, and hence $a_{t+1}\neq a_t^{\min}$. Finally, defining $M_{\sf ITCB}\triangleq\max\{M_3^\epsilon,M_4^{\epsilon}\}$, it satisfies $\frac{1}{t}T_t(a_t^{\min})\leq \gamma_{t,a_t^{\min}}$.
    \end{itemize}

\end{enumerate}

\end{proof}

\section{Proof of Theorem~\ref{theorem: SC upper bound}}
\label{proof: SC upper bound}

The upper bound on the average sample complexity is obtained by leveraging the convergence of the empirical problem complexity $\Gamma_t\left(\frac{1}{t}\bT_t\right)$ as a result of our sampling strategy, to the true value $\Gamma(\bnu)$, where we have defined $\bT_t\triangleq [T_t(1),\cdots,T_t(K)]$. This is stated in Lemma~\ref{lemma:scub_1}. 
%Next, Lemma~\ref{lemma:scub_2} states the relationship between the GLLR $\Lambda_t(a_t^{\sf top},a_t^{\sf ch})$ and the empirical problem complexity $\Gamma_t\left(\frac{1}{t}\bT_t\right)$. Finally, leveraging Lemma~\ref{lemma:scub_1} and Lemma~\ref{lemma:scub_2}, we obtain an upper bound on the average sample complexity as a result of the TCB and ITCB sampling rules~(\ref{eq: sampling rule}) and~(\ref{eq: sampling rule 2}), and the stopping rule~(\ref{eq:stop}).

\begin{lemma}
\label{lemma:scub_1}
    Under TCB and ITCB, for any $\epsilon\in\R_+$, there exists $N_{\epsilon}$ such that for all $t\geq N_\epsilon$, we have
    \begin{align}
        \left\lvert \Gamma(\bnu) - \Gamma_t\left(\frac{1}{t}\bT_t\right)\right\rvert\;\leq\;\epsilon\ ,
    \end{align}
   and $\E_{\bnu}[N_{\epsilon}]<+\infty$.
\end{lemma}
\begin{proof}
    For any $\epsilon^\prime>0$, let us define the time $N_1^{\epsilon^\prime}\triangleq \max\{N_{\bnu}^{\epsilon^\prime}, N_{\bw}^{\epsilon^\prime}\}$. For all $t>N_1^{\epsilon^\prime}$, we have:
    \begin{enumerate}
        \item $\mu_t(i)\in[\mu(i)-\epsilon^\prime,\mu(i)+\epsilon^\prime]$ for every arm $i\in[K]$.
        \item $\frac{1}{t}T_t(i)\in[w_i(\bnu)-\epsilon^\prime,w_i(\bnu)+\epsilon^\prime]$ for every arm $i\in[K]$.
        \item Let $\bnu_t\triangleq[\P_{t,1},\cdots,\P_{t,K}]$ denote the bandit instance characterized by the mean values $m(\bnu_t) = [\mu_t(1),\cdots,\mu_t(K)]$. As a result of the continuity of $\Gamma(\bnu,\bw)$ in its first argument established in Lemma~\ref{lemma:properties of problem complexity}, for all $t>N_1^{\epsilon^\prime}$, there exists $\epsilon^{\prime\prime}$ such that we have $|\Gamma(\bnu,\bw)-\Gamma(\bnu_t,\bw)|<\epsilon^{\prime\prime}$.
        % It can be readily verified using the mean value theorem there exists a universal constant $c_1\in\R_+$ such that for every $i\in[K]$, for any $\theta,\theta^\prime\in\Theta$ and any $x\in\Omega$ we have
        % \begin{align}
        % \label{eq:mean_value}
        %     |\pi_i(x\med\theta) - \pi_i(x\med\theta^\prime)|\;\leq\;c_1|\theta-\theta^{\prime}|\ .
        % \end{align}
        % Leveraging~(\ref{eq:mean_value}), it can be readily verified that there exists a universal constant $c_2\in\R_+$ such that for every $i\in[K]$, we have
        % \begin{align}
        %     |d_{\sf U}(\P_{t,i},x)-d_{\sf U}(\P_i,x)|\;\leq\;c_2\epsilon^\prime,\quad\text{and}\qquad |d_{\sf L}(\P_{t,i},x)-d_{\sf L}(\P_i,x)|\;\leq\;c_2\epsilon^\prime\ .
        % \end{align}
    \end{enumerate}
    Thus, for all $t>N_1^{\epsilon^\prime}$, we have
    \begin{align}
    \label{eq:scub_1}
        \Gamma_t\left(\frac{1}{t}\bT\right)\;&=\; \min\limits_{i\neq a^\star}\min\limits_{x\in I_{t,i}}\;\left\{ \frac{T_t(a^\star)}{t}d_{\sf U}(\P_{t,a^\star},x) + \frac{T_t(i)}{t}d_{\sf L}(\P_{t,i},x)\right\}\\
        \label{eq:scub_2_1}
        &\leq\; \min\limits_{i\neq a^\star}\; \min\limits_{x\in I_{t,i}}\;\left\{ \left(w_{a^\star}(\bnu) + \epsilon^\prime\right) d_{\sf U}(\P_{t,a^\star},x) + \frac{T_t(i)}{t}d_{\sf L}(\P_{t,i},x)\right \}\\
        \label{eq:scub_2_2}
        &\leq\; \min\limits_{i\neq a^\star}\; \min\limits_{x\in I_{t,i}}\;\left\{ \left(w_{a^\star}(\bnu) + \epsilon^\prime d_{\sf U}(\P_{t,a^\star},x)\right) + \left(w_i(\bnu) + \epsilon^\prime\right)d_{\sf L}(\P_{t,i},x)\right \}\\
        &=\;\Gamma(\bnu_t,\bw(\bnu)) + O(\epsilon^\prime)\\
        &\leq\;\Gamma(\bnu) + \underbrace{\epsilon^{\prime\prime} + O(\epsilon^\prime)}_{\triangleq\;\epsilon}\ ,
        \label{eq:scub_3}
    \end{align}
    where~(\ref{eq:scub_1}) follows from the fact that $a_t^{\sf top}=a^\star$ for all $t>N_1^{\epsilon^\prime}$,~(\ref{eq:scub_2_1}) and~(\ref{eq:scub_2_2}) follow from the fact that $\Gamma_i(\bnu,\bw)$ is an increasing function in each coordinate $w_i$, keeping the other coordinates fixed~\cite[Lemma 2]{russo2016}, and~(\ref{eq:scub_3}) follows from the fact that $t>N_1^{\epsilon^\prime}$. Following similar steps as~(\ref{eq:scub_1})-(\ref{eq:scub_3}), we can show that 
    \begin{align}
        \Gamma_t\left(\frac{1}{t}\bT_t\right) \geq \Gamma(\bnu) - \epsilon\ .
    \end{align}
    The proof is completed by setting $N_{\epsilon}\triangleq N_1^{\epsilon^\prime}$.
\end{proof}

\noindent Next, we investigate the relationship between the empirical problem complexity $\Gamma_t\left(\frac{1}{t}\bT_t\right)$ and the normalized test statistic $\frac{1}{t}\Lambda_t(a_t^{\sf top},a_t^{\sf ch})$. For this, we leverage the convergence of the test statistic to the log-likelihood ratio. For any arm $i\in[K]$ and parameters $\theta,\theta^\prime\in\Theta$, let us define
\begin{align}
    {\sf nLLR}_t(i,\theta,\theta^\prime)\;\triangleq\; \frac{1}{T_t(i)}\sum\limits_{s\in[t]: A_s = i} \log\frac{\pi_i(X_s\med\theta)}{\pi_i(X_s\med\theta^\prime)}\ .
\end{align}
Furthermore, for any $\epsilon\in\R_+$ let us define the time instant
\begin{align}
    N_{\sf KL}^\epsilon(i,\theta,\theta^\prime)\;\triangleq\;\sup\left\{ t\in\N : \left\lvert {\sf nLLR}_t(i,\theta,\theta^\prime) - d_i(\theta\|\theta^\prime)\right\rvert>\epsilon\right\}\ .
\end{align}
Together with Assumption 8, and following the same line of arguments as Theorem~\ref{theorem: convergence in mean}, we have
\begin{align}
    \E_{\bnu}\left[N_{\sf KL}^\epsilon(i,\theta,\theta^\prime)\right]\;<\;+\infty\ .
\end{align}
Let us define 
\begin{align}
   M_{\sf KL}^\epsilon & \triangleq\max_{i\in[K]}N_{\sf KL}^\epsilon(i,\mu_t(i),\tilde\mu_t(i))\ ,\\
   \bar M_{\sf KL}^\epsilon & \triangleq\max_{i\in[K]}N_{\sf KL}^\epsilon(i,\bar\mu_t(i),\tilde\mu_t(i))\ , \\
   \bar N_{\sf KL}^\epsilon & \triangleq \max\{M_{\sf KL}^\epsilon,\bar M_{\sf KL}^\epsilon, N_{\bnu}^{\Delta_{\min}/4}\}\ .
\end{align}
For all $t>\bar N_{\sf KL}^\epsilon$, we have
\begin{align}
    \frac{1}{t}\Lambda_t(a_t^{\sf top},a_t^{\sf ch})\;&=\;\min\limits_{i\neq a^\star}\frac{1}{t}\Lambda_t(a^\star,i)\\
    &=\; \min\limits_{i\neq a^\star}\left\{ \frac{T_t(a^\star)}{t} d_{a^\star}(\mu_t(a^\star)\|\tilde\mu_t(i)) + \frac{T_t(i)}{t}d_i(\mu_t(i),\tilde\mu_t(i))\right\}\\
    &\geq\; \min\limits_{i\neq a^\star}\left\{ \frac{T_t(a^\star)}{t}\left( {\sf nLLR}_t(a^\star,\mu_t(a^\star),\tilde\mu_t(i))-\epsilon\right) + \frac{T_t(i)}{t}\left({\sf nLLR}_t(i,\mu_t(i),\tilde\mu_t(i))-\epsilon \right)\right\}\\
    &\geq\; \min\limits_{i\neq a^\star}\left\{ \frac{T_t(a^\star)}{t}{\sf nLLR}_t(a^\star,\mu_t(a^\star),\tilde\mu_t(i)) + \frac{T_t(i)}{t}{\sf nLLR}_t(i,\mu_t(i),\tilde\mu_t(i))\right\} - \epsilon\\
    &\geq\; \min\limits_{i\neq a^\star}\left\{ \frac{T_t(a^\star)}{t}{\sf nLLR}_t(a^\star,\bar\mu_t(a^\star),\tilde\mu_t(i)) + \frac{T_t(i)}{t}{\sf nLLR}_t(i,\bar\mu_t(i),\tilde\mu_t(i))\right\} - \epsilon\\
    &\geq\;\min\limits_{i\neq a^\star}\left\{ \frac{T_t(a^\star)}{t}\left( d_{a^\star}(\bar\mu_t(a^\star)\|\tilde\mu_t(i))-\epsilon\right) + \frac{T_t(i)}{t}\left( d_i(\bar\mu_t(i)\|\tilde\mu_t(i))-\epsilon\right)\right\} - \epsilon\\
    &\geq\;\min\limits_{i\neq a^\star}\left\{ \frac{T_t(a^\star)}{t} d_{a^\star}(\bar\mu_t(a^\star)\|\tilde\mu_t(i)) + \frac{T_t(i)}{t} d_i(\bar\mu_t(i)\|\tilde\mu_t(i))\right\} - 2\epsilon\\
    &\geq\;\min\limits_{i\neq a^\star}\min\limits_{x\in\R}\left\{ \frac{T_t(a^\star)}{t} d_{a^\star}(\bar\mu_t(a^\star)\|x) + \frac{T_t(i)}{t} d_i(\bar\mu_t(i)\|x)\right\} - 2\epsilon\\
    &\geq\;\min\limits_{i\neq a^\star}\min\limits_{x\in I_{t,i}}\left\{ \frac{T_t(a^\star)}{t} d_{\sf U}(\P_{t,a^\star},x) + \frac{T_t(i)}{t} d_{\sf L}(\P_{t,i},x)\right\} - 2\epsilon\\
    &=\;\Gamma_t\left( \frac{1}{t}\bT_t\right) - 2\epsilon\ .
    \label{eq:Gamma_and_lambda}
\end{align}

Next, we proceed with the proof of Theorem~\ref{theorem: SC upper bound}. 
% For all $t>N_1{\epsilon^\prime}$, the test statistic $\Lambda_t(a_t^{\sf top},a_t^{\sf ch})$ satisfies
% \begin{align}
%     \frac{1}{t}\Lambda_t(a_t^{\sf top}, a_t^{\sf ch})\;&=\;\min\limits_{i\neq a^\star}\left\{ \frac{T_t(a^\star)}{t} d_{a^\star}(\mu_t(a^\star)\|\tilde\mu_t(i)) + \frac{T_t(i)}{t}d_i(\mu_t(i)\|\tilde\mu_t(i))\right\}\\
%     &\geq\;\min\limits_{i\neq a^\star}\left\{ \frac{T_t(a^\star)}{t} d_{a^\star}(\bar\mu_t(a^\star)\|\tilde\mu_t(i)) + \frac{T_t(i)}{t}d_i(\bar\mu_t(i)\|\tilde\mu_t(i))\right\}\\
%     &\geq\; \min\limits_{i\neq a^\star}\min\limits_{x\in I_{t,i}}\left\{\frac{T_t(a^\star)}{t}d_{\sf U}(\P_{t,a^\star},x) + \frac{T_t(i)}{t}d_{\sf L}(\P_{t,i},x)\right\}\\
%     &=\;\Gamma_t\left( \frac{1}{t}\bT_t\right)\ .
% \end{align}
%Next, in order to prove Theorem~\ref{theorem: SC upper bound}, let us define the time instant $N_3^\epsilon\triangleq \max\{N_{\epsilon},N_2^\epsilon\}$. 
Expanding the time instance just before stopping, we have
\begin{align}
    \tau - 1 \;&=\; (\tau - 1)\mathds{1}_{\{\tau-1\leq N_2^\epsilon\}} + (\tau-1)\mathds{1}_{\{\tau-1>N_2^\epsilon\}}\\
    &\leq\; N_2^\epsilon + (\tau-1)\mathds{1}_{\{\tau-1>N_2^\epsilon\}}\ ,
    \label{eq:th_SCUB_N2e}
\end{align}
where we have defined $N_2^\epsilon\triangleq\max\{N_{\epsilon},\bar N_{\sf KL}^\epsilon, N_2\}$, and $N_2$ will be specified later. Leveraging Lemma~\ref{lemma:scub_1}, along with the fact that at $\tau-1$, $\Lambda_{\tau-1}(a_{\tau-1}^{\sf top},a_{\tau-1}^{\sf ch})\leq \beta_{\tau-1}(\delta)$, if $\tau-1>N_2^\epsilon$ we have
\begin{align}
    \Gamma({\bnu}) - 3\epsilon\;\leq\;\Gamma_{\tau-1}\left(\frac{1}{\tau-1}\bT_{\tau-1}\right) -2\epsilon\;\stackrel{\eqref{eq:Gamma_and_lambda}}{\leq}\; \frac{\Lambda_{\tau-1}(a_{\tau-1}^{\sf top},a_{\tau-1}^{\sf ch})}{\tau-1}\;\leq\; \frac{\beta_{\tau-1}(\delta)}{\tau-1}\ .
\end{align}
Furthermore, leveraging the stopping threshold $\beta_t(\delta)$ defined in Theorem~\ref{theorem: delta-PAC}, if $\tau-1>N_2^\epsilon$ we have
\begin{align}
    (\tau-1)\Big( \Gamma(\bnu) - 3\epsilon\Big)\;&\leq\; \max\limits_{i\in[K]}\log\mcI_i(\mu_{\tau-1}(i)) - 2\cdot\min\limits_{i\in[K]} W_{\tau -1}(i)\nonumber\\
    &\qquad + (\tau-1)\cdot\max\limits_{i\in[K]}\left\{ \max\{d_i(\mu_{\tau-1}(i)\|\mu_{\tau-1}(i)-\varepsilon),d_i(\mu_{\tau-1}(i)\|\mu_{\tau-1}(i)+\varepsilon)\}\right\}\nonumber\\
    &\qquad + 2\log\frac{|\Theta|}{\sqrt{2\pi}} + \log\frac{\tau-1}{2} + \log\frac{K-1}{\delta}\ .
    \label{eq:th_SCUB_1}
\end{align}
Let us define
\begin{align}
 \mcI_{\max}\;\triangleq\;\max\limits_{i\in[K]}\max\limits_{\theta\in\Theta}\mcI_i(\theta)\ .
    \label{eq:th_SCUB_2}
\end{align}
Furthermore, for any $t\in\N$ and for all $i\in[K]$ we have
\begin{align}
    W_t(\varepsilon,i)\;&=\; \displaystyle\bigintsss_{\Omega^{\otimes T_t(i)}} \log\left( 1-2Q\left( \varepsilon\sqrt{\bar V_t(i)}\right)\right)\displaystyle\prod\limits_{s\in[t]:A_s = i} \pi_i(X_s\med\mu_t(i))\diff \mcX_t^i\\
    \label{eq:th1_SCUB_2b}
    \log\left( 1-2Q\left( \varepsilon\sqrt{T_t(i)\mcI_i(\mu_t(i))}\right)\right)\;&\geq\;\log\left( 1-2Q\left( \varepsilon\sigma\sqrt{T_t(i)}\right)\right)\\
    &\geq\;\log\left( 1-2Q\left( \varepsilon\sigma\sqrt{\sqrt{t/K}-1}\right)\right)\ ,
    \label{eq:th1_SCUB_2a}
\end{align}
where~\eqref{eq:th1_SCUB_2b} is obtained using Assumption 7, and~\eqref{eq:th1_SCUB_2a} is follows from the explicit exploration property of the TCB and ITCB algorithms. Hence, defining
\begin{align}
    N_2\;\triangleq\; K\left(\left( \frac{1}{\varepsilon\sigma}Q^{-1}\left(\frac{1}{4}\right)\right)^2 + 1\right)^2\ ,
\end{align}
for all $t>N_2$ and for all $i\in[K]$ we have
\begin{align}
    -\log\left( 1-2Q\left( \varepsilon\sigma\sqrt{T_t(i)}\right)\right)\;&\leq\;\log 2\ .
    \label{eq:th_SCUB_3}
\end{align}
Next, note that as a result of Assumption~5, we have
\begin{align}
    \max\limits_{i\in[K]}\left\{ \max\{d_i(\mu_{\tau-1}(i)\|\mu_{\tau-1}(i)-\varepsilon),d_i(\mu_{\tau-1}(i)\|\mu_{\tau-1}(i)+\varepsilon)\}\right\}\;\leq\; \zeta(\varepsilon)\ ,
    \label{eq:th_SCUB_4}
\end{align}
where $\zeta(\varepsilon)\in\R_+$ can be made arbitrarily small by appropriately choosing $\varepsilon$. Hence, using~\eqref{eq:th_SCUB_2},~\eqref{eq:th_SCUB_3} and~\eqref{eq:th_SCUB_4}, for $\tau_1>N_2^\epsilon$ we can rearrange~\eqref{eq:th_SCUB_1} as follows.
\begin{align}
    (\tau-1)\Big( \Gamma(\bnu) - 3\epsilon - \zeta(\varepsilon)\Big)\;&\leq\; \log\mcI_{\max} + 2\log\sqrt{\frac{2}{\pi}}|\Theta| + \log\frac{\tau-1}{2} + \log\frac{K-1}{\delta}\ .
    \label{eq:th_SCUB_5}
\end{align}
To upper bound~\eqref{eq:th_SCUB_5}, we leverage~\cite[Lemma 18]{pmlr-v49-garivier16a} which gives
\begin{align}
    \tau\;\stackrel{\eqref{eq:th_SCUB_N2e}}{\leq}\; N_2^\epsilon + \frac{1}{\Big(\Gamma(\bnu) - \epsilon - \zeta(\varepsilon)\Big)}\cdot\left( \log\frac{(K-1)\mcI_{\max}|\Theta|^2\e}{\Big(\Gamma(\bnu) - \epsilon - \zeta(\varepsilon)\Big)\delta} + \log\log \frac{(K-1)\mcI_{\max}|\Theta|^2}{\Big(\Gamma(\bnu) - \epsilon - \zeta(\varepsilon)\Big)}\right) + 1\ .
\end{align}
Next, taking expectation on both sides, dividing by $\log(1/\delta)$ and taking the limit of $\delta\rightarrow 0$, we have
\begin{align}
    \lim\limits_{\delta\rightarrow 0}\;\frac{\E_{\bnu}[\tau]}{\log(1/\delta)}\;\leq\; \frac{1}{\Big(\Gamma(\bnu) - 3\epsilon - \zeta(\varepsilon)\Big)}\ .
    \label{eq:scub_6}
\end{align}
Taking infimum with respect to $\epsilon$ in~\eqref{eq:scub_6}, we have
\begin{align}
    \lim\limits_{\delta\rightarrow 0}\;\frac{\E_{\bnu}[\tau]}{\log(1/\delta)}\;&\leq\; \frac{1}{\Big(\Gamma(\bnu) - \zeta(\varepsilon)\Big)}\\
    &=\; \frac{1+\alpha}{\Gamma(\bnu)}\ ,
\end{align}
where we have defined $\alpha\triangleq \zeta(\varepsilon)/(\Gamma(\bnu)-\zeta(\varepsilon))$, and $\alpha$ can be made arbitrarily small by choosing a sufficiently small $\varepsilon$. This completes the proof.

\bibliographystyle{IEEEbib}
\bibliography{BAIRef}

\end{document}